\documentclass[10pt,letterpaper]{article}

\usepackage[top=1in,bottom=1in,left=1in,right=1in]{geometry}

\usepackage[utf8]{inputenc} 
\usepackage[T1]{fontenc}    
\usepackage{hyperref}       
\usepackage{url}            
\usepackage{booktabs}       
\usepackage{amsfonts}       
\usepackage{nicefrac}       
\usepackage{microtype}      
\usepackage{xcolor}         

\usepackage{graphicx} 
\usepackage{amsmath, amssymb}
\usepackage{hyperref}
\usepackage{amsthm}
\usepackage{cleveref} 
\usepackage{color}
\usepackage{subcaption}

\newtheorem{theorem}{Theorem}

\newtheorem{lemma}{Lemma}
\newtheorem{corollary}{Corollary}
\newtheorem{remark}{Remark}

\crefname{definition}{Definition}{Definitions}
\crefname{lemma}{Lemma}{Lemmas}
\crefname{theorem}{Theorem}{Theorems}
\crefname{corollary}{Corollary}{Corollaries}
\crefname{figure}{Figure}{Figures}
\crefname{section}{Section}{Sections}
\crefname{remark}{Remark}{Remarks}

\DeclareMathOperator*{\argmax}{arg\,max}
\newcommand{\Prob}{\mathbb{P}}
\newcommand{\Exp}{\mathbb{E}}
\newcommand{\R}{\mathbb{R}}
\newcommand{\inv}[1]{{#1}^{-1}}
\newcommand{\norm}[1]{{\left\| #1 \right\|}}
\newcommand{\innerprod}[2]{{\left\langle #1, #2 \right\rangle}}
\newcommand{\sign}{{\rm sign}}
\newcommand{\Diag}{{\rm Diag}}
\newcommand{\zero}{\mathbf{0}}
\newcommand{\one}{\mathbf{1}}
\newcommand{\A}{\mathbf{A}}
\newcommand{\E}{\mathbf{E}}
\newcommand{\D}{\mathbf{D}}
\renewcommand{\L}{\mathbf{L}}
\newcommand{\M}{\mathbf{M}}
\newcommand{\F}{\mathbf{F}}
\renewcommand{\H}{\mathbf{H}}
\newcommand{\X}{\mathbf{X}}
\renewcommand{\v}{\mathbf{v}}
\newcommand{\x}{\mathbf{x}}
\newcommand{\y}{\mathbf{y}}
\newcommand{\e}{\mathbf{e}}
\newcommand{\yt}{{\y^*}}
\newcommand{\yh}{\hat{\y}}
\newcommand{\avg}{\mathrm{avg}}
\renewcommand{\a}{\mathbf{a}}
\renewcommand{\b}{\mathbf{b}}
\renewcommand{\c}{\mathbf{c}}
\renewcommand{\S}{\mathbb{S}}
\newcommand{\s}{\mathbf{s}}

\title{\textbf{Provable Guarantees for Spectral Structured Prediction}}

\author{Violet Zheng\\
CIS, The University of Melbourne\\
\texttt{violet.zheng@student.unimelb.edu.au}\\
\and
Jean Honorio\\
CIS, The University of Melbourne\\
\texttt{jean.honorio@unimelb.edu.au}}

\date{}

\begin{document}

\maketitle

\begin{abstract}
Structured prediction is the simultaneous prediction of multiple labels, and is widely used in various fields, such as natural language processing and computer vision.
In this paper, we study binary node label recovery on signed graphs with edge-flip noise, a model introduced by~\cite{globerson_2015}, via a simple spectral method that decodes node labels from the signs of the principal eigenvector of the noisy signed adjacency matrix. 
We develop graph structure-agnostic theoretical guarantees for approximate inference of node labels as well as guarantees for maximum angle deviation with respect to the ground truth node labels.
By leveraging tools from matrix concentration theory and eigenvector perturbation analysis, we derive new concentration inequalities that explicitly quantify the effect of the spectral gap of the adjacency matrix, number of nodes, degree distribution, and noise level.
As a corollary, we relate our general results to the Cheeger constant and provide results for different classes of graphs.
We perform several synthetic experiments to validate our theory.
To the best of our knowledge, we are the first to provide theoretical guarantees for the spectral-based approach.
As a byproduct of our analysis, we derive technical results that might be of independent interest and useful for other machine learning problems.
\end{abstract}

\section{Introduction} \label{sec:intro}

Inference in complex models drives much of the research in machine learning applications in many diverse domains, such as natural language processing, computer vision and speech processing.
Structured prediction refers to the prediction of complex structures from a given input.
Structured prediction has been shown to be useful in applications such as dependency parsing~\cite{Martins13,Zhang14}, part-of-speech tagging~\cite{Zhang15}, object detection~\cite{Felzenszwalb10}, scene understanding~\cite{Lee10,Wang10b}, phoneme/speech recognition~\cite{Keshet11,Zhang13b} and text-to-speech mapping~\cite{Tang12}.

In the context of Markov random fields (MRFs), for an undirected graph $G = (V,E)$ of $n$ nodes, one is interested in finding a solution to the following inference problem:
\begin{align}
\label{eq:mrf_inference}
\max_{\y \in L^n} & \sum_{i \in V, l \in L} \phi_i(m) 1[y_i=l] 
+ \sum_{(i,j) \in E, l,m \in L} \hspace{-0.15in} \phi_{ij}(l,m) 1[y_i = l, y_j = m],
\end{align}
where $L$ is the set of possible labels, $\phi_i(l)$ is the (unary) utility of assigning label $l$ to node $i$, and
$\phi_{ij}(l,m)$ is the (pairwise) utility of assigning labels $l$ and $m$ to the neighbors $i$ and $j$, respectively.
Very few cases of the optimization problem in \cref{eq:mrf_inference} are known to be solvable in polynomial time.
For example, \cref{eq:mrf_inference} is solvable in polynomial time for graphs with low treewidth via the junction tree algorithm~\cite{chandrasekaran2008complexity}.
For binary labels, i.e., $|L|=2$, \cref{eq:mrf_inference} is solvable in polynomial time for planar graphs via perfect matchings~\cite{schraudolph2009efficient}.
Finally, \cref{eq:mrf_inference} is solvable in polynomial time via graph cuts for binary labels and submodular pairwise potentials~\cite{boykov2006graph}.

Given the above, prior research has focused on theoretical guarantees for binary labels, i.e., $|L|=2$, on various classes of graphs for the correct recovery of all node labels~\cite{bello_2019} and the correct recovery of a fraction of node labels~\cite{foster2018inference,globerson_2015,ke_2025}.
Most importantly, the foundational paper of~\cite{globerson_2015} defines a model for the study of binary node label recovery on signed graphs with edge-flip noise.
This serves as a meaningful special case of \cref{eq:mrf_inference} that allows to analyze the impact of structural properties of graphs on inference.

Our contributions are as follows.
We study binary node label recovery on signed graphs with edge-flip noise~\cite{globerson_2015}, via a simple spectral method.
We develop graph structure-agnostic theoretical guarantees for approximate inference of node labels and for maximum angle deviation with respect to the ground truth.
To the best of our knowledge, we are the first to provide theoretical guarantees for the spectral-based approach as well as angle deviation results within the inference in structured prediction context.

Interestingly, our success rates resemble those of~\cite{bello_2019,ke_2025} but our spectral method is arguably more computationally efficient than semidefinite programming.
Our spectral method can be applied to various graph classes, and it is not limited to planar graphs~\cite{globerson_2015} or graphs with low treewidth~\cite{foster2018inference}.
Finally, we provide results for different classes of graphs and perform several synthetic experiments to validate our theory.

\section{Preliminaries} \label{sec:prelim}

In this section, we introduce the main model and notations that are used throughout the paper.

\subsection{Notation}

We use lowercase bold letters for vectors (e.g., $\b$) and uppercase bold letters for matrices (e.g., $\M$), and subscripts of non-bold letters to denote respective entries of a vector or matrix (e.g., $b_i$ and $M_{ij}$).
The Euclidean norm of a vector $\b$ is denoted as $\norm{\b}$.
The inner product of two vectors $\a$ and $\b$ is denoted as $\innerprod{\a}{\b}$.
A vector of all ones (or all zeros) is denoted as $\one$ (or $\zero$).
The spectral norm of a matrix $\M$ is denoted as $\norm{\M}$.
For a matrix $\M \in \R^{n \times n}$, we denote the dominant eigenvalue as $\lambda_1(\M)$, second dominant eigenvalue as $\lambda_2(\M)$ and so on, with the smallest eigenvalue being $\lambda_n(\M)$.
That is, $\lambda_1(\M) \geq \lambda_2(\M) \geq ... \geq \lambda_n(\M)$.
Meanwhile, $\v_1(\M)$ refers to the eigenvector associated with $\lambda_1(\M)$.
For a vector $\b$, $\Diag(\b)$ is the matrix containing $\b$ on its diagonal, and zero everywhere else.
The Hadamard product $\M = \F \circ \H$ represents an entrywise multiplication, i.e., $M_{ij} = F_{ij} H_{ij}$.

We assume graph $G=(V,E)$ has node set $V=\{1,\dots,n\}$ and edge set $E \subseteq V \times V$.
For convenience, we also represent the edge set as an adjacency matrix $\E$.
That is, $(i,j) \in E$ iff $E_{ij}=1$, and $(i,j) \notin E$ iff $E_{ij}=0$.
Matrix $\D$ is a diagonal degree matrix with $D_{ii}=d_i$ where $d_i = |\{ j \mid (i,j) \in E \}| = \sum_{j=1}^n E_{ij}$ is the degree of the $i$-th node.
The Laplacian of graph $G$ is $\L = \D-\E$.
In addition, $d_{\max} = \max_{i=1}^n d_i$ refers to the maximum node degree and $d_\avg = \frac{1}{n} \sum_{i=1}^n d_i$ refers to the average node degree.
We also define $d_{\Delta} = d_{\max}-d_\avg$.

Our main results focus on both angle deviations and the proportion of sign agreements.
We define $\theta : \R^n \times \R^n \to [0,\pi]$ to be a function that returns the angle between two vectors.
More formally, $\cos\theta(\a,\b) = \innerprod{\frac{\a}{\norm{\a}}}{\frac{\b}{\norm{\b}}}$.
We also define $\rho : \R^n \times \R^n \to [0,1]$ to be a function that returns the proportion of sign agreements between two vectors.
More formally, $\rho(\a,\b) = \frac{1}{n} \max\left( \sum_{i=1}^n 1[a_i b_i > 0], \, \sum_{i=1}^n 1[-a_i b_i > 0] \right)$.

\subsection{Inference in Structured Prediction} \label{sec:model}

Next, we present the model introduced by~\cite{globerson_2015}, for the study of binary node label recovery on signed graphs with edge-flip noise.
Consider a graph $G = (V, E)$ with node set $V = \{1,\dots,n\}$ and edge set $E \subseteq V \times V$.
Each node $i \in V$ has a ground truth node label $y_i^* \in \{-1,+1\}$.
We also write $\yt \in \{-1,+1\}^n$ to refer to the ground truth node labels for all nodes.

A noisy signed adjacency matrix $\A$ and noisy node observation vector $\c$ are generated as follows, using noise parameters $p,q \in [0,\frac{1}{2})$.
$A_{ij} = 0$ if there is no edge between nodes $i$ and $j$ in $G$, otherwise $A_{ij} = y_i^* y_j^*$ with probability $1 - p$ and $A_{ij} = -y_i^* y_j^*$ with probability $p$.
$c_i=y_i^*$ with probability $1-q$ and $c_i=-y_i^*$ with probability $q$.

The inference problem seeks to recover the true node-label vector $\yt$ from the given $\A$ and $\c$.
As the noise level $p$ (or $q$) increases, exact recovery becomes more difficult.
The maximum likelood estimator (MLE) for exact recovery of the node labels is:
\[
\max_{\y \in \{-1,+1\}^n} \frac{1}{2}\y^\top \A \y+\frac{\log (1-q)/q}{\log (1-p)/p} \c^\top \y ,
\]
which is a special case of \cref{eq:mrf_inference}.
Unfortunately, since the noise parameters $p,q$ are unknown and so the factor $\frac{\log (1-q)/q}{\log (1-p)/p}$, one cannot solve the above problem directly.
To tackle this issue, \cite{globerson_2015} introduced a two-stage approach, loosely based on the MLE.
A first stage solves the following optimization problem:
\begin{align} \label{eq:pairwise}
\yh = \argmax_{\y \in \{-1,+1\}^n} \frac{1}{2}\, \y^\top \A \y,
\end{align}
for which two possible results, $\yh$ and $-\yh$ are equivalent in the sense of their objective function value being the same.\footnote{Note that $\yh^\top \A \yh = (-\yh)^\top \A (-\yh)$, and thus we focus our analysis on $|\cos\theta(\yh,\yt)|=|\cos\theta(-\yh,\yt)|$, and defined the proportion of sign agreements such that $\rho(\yh,\yt)=\rho(-\yh,\yt)$.}
A second (disambiguation) stage decides which of the two solutions ($\yh$ or $-\yh$) leads to the optimal solution of the following problem:
\begin{align} \label{eq:unary}
\max_{\y \in \{\yh,-\yh\}} \c^\top \y.
\end{align}
The second stage was already analyzed in~\cite{globerson_2015,bello_2019}. This allows us to concentrate only on the first-stage optimization problem in \cref{eq:pairwise} in this paper by proposing a spectral approach.

\section{Main Results} \label{sec:main}

In this section, we provide the algorithm, proofs and statements of our main results.
We present these findings in the order of significance, with intermediate lemmas used in proofs presented at the end of the section. 

Our spectral method uses the dominant eigenvector $\v_1(\A)$ of the noisy signed adjacency matrix $\A$, in order to decide which labels are assigned to the nodes.
For each node $i$, we declare $y_i=-1$ if $\v_1(\A)_i<0$, and $y_i=+1$ if $\v_1(\A)_i>0$.
There exist highly efficient methods to compute the dominant eigenvector~\cite{xu_2020_complexity}.
For $n$ nodes and $|E|$ edges, computing the dominant eigenvector takes $O(n+|E|)$ time, which is very efficient for sparse graphs.
For dense graphs, $|E| \in O(n^2)$ and thus, computing the dominant eigenvector takes $O(n^2)$ time. 

Our main result in \cref{thm:angle} explicitly provides a concentration inequality for the angle deviation between $\v_1(\A)$ and $\yt$, while \cref{thm:proportion} connects this angle deviation discovery to a probabilistic guarantee in the proportion of correctly recovered node labels.

\subsection{Angle Deviation Guarantees}

We start by stating our first main result.

\begin{theorem} \label{thm:angle}
For any $0<\gamma<1$, the probability that $\v_1(\A)$ aligns closely with $\yt$ is lower bounded by some function dependent on the number of nodes $n$, the maximum node degree $d_{\max}$, the edge-noise parameter $p$, and the dominant eigen gap of $\E$, i.e., $\lambda_1(\E)-\lambda_2(\E)$.
More formally, we have that
\begin{align*}
&\Prob\left(|\cos\theta(\v_1(\A), \yt)| \ge 1-\gamma \right) \\
&\!\ge 1 - 2n \exp\!\left( \frac{ -C_{\delta,\gamma}
(1-2p)^2(\lambda_1(\E)-\lambda_2(\E))^2 }{ 32(1-p)^2 d_{\max} } \right) 
\!- 2n \exp\!\left( \frac{ -(1-2p)^2(\lambda_1(\E)-\lambda_2(\E))^2 }{ 128(1-p)^2 d_{\max} } \right) ,
\end{align*}
where $\delta = |\cos\theta(\v_1(\E), \one)|$ and $C_{\delta,\gamma} = 1-\delta + \delta \gamma - \sqrt{ \gamma(\delta^2 - 1)(\gamma - 1) }$ is a nonnegative constant.
\end{theorem}
(Omitted proofs are included in Appendix~\ref{app:proofs}.)
\begin{proof}[Proof sketch]
We split $|\cos\theta(\v_1(\A), \yt)|$ into two terms: one random variable $|\cos\theta(\v_1(\A), \v_1(\Exp \A))|$ and one constant $|\cos\theta(\v_1(\Exp \A), \yt)|$.
To bound the random variable, we use Davis-Kahan $\sin \theta$ Theorem~\cite{Kahan_1970} and \cref{lem:concentration}.
To bound the constant, we use \cref{lem:similar}.
Finally, to ensure the uniqueness of $\v_1(\A)$ (up to sign), we need to ensure that $\lambda_1(\A)-\lambda_2(\A)>0$, which occurs with the probability bound in \cref{lem:eigengap}.
\end{proof}

\begin{remark}
The dominant eigenvector $\v_1(\E)$ captures the centrality of each node, i.e., their recursive importance.
Since $\one$ represents the case where all nodes are equally central/important, $\delta = |\cos\theta(\v_1(\E), \one)|$ in Theorem~\ref{thm:angle} quantitatively measures how "equally central" nodes are.
\end{remark}

Next, we provide a result using the Fiedler value, i.e., second minimum eigenvalue of the Laplacian $\lambda_{n-1}(\L)$.
Our main motivation for this is the vast literature regarding properties of a graph $G$ in terms of the Cheeger constant $h_G$.
In particular, Cheeger's inequality implies that $\lambda_{n-1}(\L) \geq \frac{h_G^2}{2d_{\max}}$~\cite{cheeger1969lower}.
Thus, the next result with respect to the Fiedler value has implications for different classes of graphs, as shown in \cref{sec:graphtypes}.

\begin{corollary} \label{cor:angle}
For any $0<\gamma<1$, the probability that $\v_1(\A)$ aligns closely with $\yt$ is lower bounded by some function dependent on the number of nodes $n$, the maximum node degree $d_{\max}$, the edge-noise parameter $p$, and the Fiedler value, i.e., $\lambda_{n-1}(\L)$.
More formally, we have that
\begin{align*}
&\Prob\left(|\cos\theta(\v_1(\A), \yt)| \ge 1-\gamma \right) \\
&\!\ge 1 - 2n \exp\!\left( \frac{ -C_{\delta,\gamma}
(1-2p)^2(\lambda_{n-1}(\L)-d_{\Delta})^2 }{ 32(1-p)^2 d_{\max} } \right) 
\!- 2n \exp\!\left( \frac{ -(1-2p)^2(\lambda_{n-1}(\L)-d_{\Delta})^2 }{ 128(1-p)^2 d_{\max} } \right),
\end{align*}
where $\delta$ and $C_{\delta,\gamma}$ are defined as in \cref{thm:angle}.
\end{corollary}
\begin{proof}
We use \cref{lem:similar} and \cref{lem:bound_eigengap} to bound (from below) the dominant eigen gap of $\E$ in \cref{thm:angle} with an expression that depends on the Fiedler value and $d_{\Delta}$.
\end{proof}

\begin{remark} \label{rmk:angle}
To the best of our knowledge, our angle deviation results are novel.
We first recall that the work of~\cite{bello_2019} focused on recovery of $\yt$ exactly, by using semidefinite programming, an arguably much less efficient algorithm than ours.
Having said that, it is quite interesting that the success rate of our \cref{cor:angle} is similar to that of~\cite{bello_2019}.

For more details and to make the comparison easier, assuming $d_{\Delta}=0$ (i.e., $d_{\avg}=d_{\max}$) our success rate is $1-O\left(n \exp\left(-\frac{(1-2p)^2 \lambda^2_{n-1}(\L)}{(1-p)^2 d_{max}}\right)\right)$.
Theorem 2 in~\cite{bello_2019} shows a success rate of $1-O\left(n\exp\left(-\frac{(1-2p)^2h^4_G}{p(1-p)d^3_{\max} + (1-2p)(1-p)h^2_G d_{\max}}\right)\right)$.
For simplicity, assume $\lambda_{n-1}(\L) = \frac{h^2_G}{d_{\max}}$, which fulfills Cheeger's inequality~\cite{cheeger1969lower} and implies $h^2_G = \lambda_{n-1}(\L)\,d_{\max}$.
Then the success rate of~\cite{bello_2019} is $1-O\left(n \exp\left(-\frac{(1-2p)^2\lambda^2_{n-1}(\L)}{p(1-p)d_{\max} + (1-2p)(1-p)\lambda_{n-1}(\L)}\right)\right)$, which can be upper bounded by the success rate of $1-O\left(n \exp\left(-\frac{(1-2p)^2\lambda^2_{n-1}(\L)}{p(1-p)d_{\max}}\right)\right)$.
Thus, our success rate and the success rate of~\cite{bello_2019} have the same behavior with respect to $n$, $\lambda^2_{n-1}(\L)$ and $d_{\max}$.
\end{remark}

\begin{remark}
Our paper focuses only on the first stage described in \cref{eq:pairwise}.
As shown in Theorem 3 in~\cite{bello_2019}, the second stage described in \cref{eq:unary} has error rate $\exp\left(-\frac{n(1-2q)^2}{2}\right)$.
To compute the total success rate for our spectral method as well as for~\cite{globerson_2015,foster2018inference,bello_2019,ke_2025}, one should subtract this error rate from the success rate of the first stage.
\end{remark}

\subsection{Correct Proportion Recovery Guarantees}

Although \cref{thm:angle} gave us guarantees that $\v_1(\A)$ and $\yt$ are close in terms of their angle deviation, we would also like to find more direct guarantees in the proportion of correctly recovered node labels.
This gives us a nice way of generalizing the recovery guarantees to not just perfect recovery but also for approximate recovery, particularly meaningful on real world applications where we can tolerate certain level of mistakes.

In what follows, we state our second main result.

\begin{theorem} \label{thm:proportion}
    For any $\frac{1}{2}<\alpha<1$, the probability that $\v_1(\A)$ agrees closely with $\yt$ is lower bounded by some function dependent on the number of nodes $n$, the maximum node degree $d_{\max}$, the edge-noise parameter $p$, and the dominant eigen gap of $\E$, i.e., $\lambda_1(\E)-\lambda_2(\E)$.
    More formally, we have that
    \begin{align*}
     & \Prob(\rho(\v_1(\A), \yt) \ge \alpha) \\
     & \!\ge 1 - 2n \exp\!\left( 
        \frac{ -C_{\delta,\alpha} (1-2p)^2(\lambda_1(\E)-\lambda_2(\E))^2 }{ 32(1-p)^2 d_{\max} } \right) 
    \!- 2n \exp\!\left( \frac{ -(1-2p)^2(\lambda_1(\E)-\lambda_2(\E))^2 }{ 128(1-p)^2 d_{\max} } \right) ,
    \end{align*}
    where $\delta = |\cos\theta(\v_1(\E), \one)|$ and $C_{\delta,\alpha} = 1-\delta + \delta (1-\sqrt{\alpha}) - \sqrt{ -(1-\sqrt{\alpha})(\delta^2 - 1)\sqrt{\alpha} }$ is a nonnegative constant.
\end{theorem}
\begin{proof}[Proof sketch]
By reasoning about the maximum angle deviation in order to guarantee a minimum proportion of correctly recovered node labels $\alpha$, we get:
\begin{align*}
\Prob(\rho(\v_1(\A), \yt) \ge \alpha) \geq \Prob\left( |\cos\theta(\v_1(\A), \yt)| \ge \sqrt{\alpha} \right) .
\end{align*}
Then, we invoke \cref{thm:angle} with $\gamma = 1-\sqrt{\alpha}$.
\end{proof}

As discussed before \cref{cor:angle}, we also provide a result using the Fiedler value, which allows us to have implications for different classes of graphs in \cref{sec:graphtypes}.

\begin{corollary} \label{cor:proportion}
    For any $\frac{1}{2}<\alpha<1$, the probability that $\v_1(\A)$ agrees closely with $\yt$ is lower bounded by some function dependent on the number of nodes $n$, the maximum node degree $d_{\max}$, the edge-noise parameter $p$, and the Fiedler value, i.e., $\lambda_{n-1}(\L)$.
    More formally, we have that
    \begin{align*}
    & \Prob(\rho(\v_1(\A), \yt) \ge \alpha) \\
    & \!\ge 1 - 2n \exp\!\left( 
        \frac{ -C_{\delta,\alpha} (1-2p)^2(\lambda_{n-1}(\L)-d_{\Delta})^2 }{ 32(1-p)^2 d_{\max} } \right) 
    \!- 2n \exp\!\left( \frac{ -(1-2p)^2(\lambda_{n-1}(\L)-d_{\Delta})^2 }{ 128(1-p)^2 d_{\max} } \right),
    \end{align*}
    where $\delta$ and $C_{\delta,\alpha}$ are defined as in \cref{thm:proportion}.
\end{corollary}
\begin{proof}
We use \cref{lem:similar} and \cref{lem:bound_eigengap} to bound (from below) the dominant eigen gap of $\E$ in \cref{thm:proportion} with an expression that depends on the Fiedler value and $d_{\Delta}$.
\end{proof}

\begin{remark} \label{rmk:proportion}
Regarding guarantees of a proportion of correctly recovered node labels, the work of~\cite{globerson_2015,foster2018inference} use tractable combinatorial algorithms that only work for specific classes of graphs: planar graphs~\cite{globerson_2015} and graphs with low treewidth~\cite{foster2018inference}.
The work of~\cite{ke_2025} uses semidefinite programming, an arguably much less efficient algorithm than ours, and provides success rates that are not easy to interpret.

Still, to make the comparison easier, one can consider $k=n$ in Theorem 1 in~\cite{ke_2025} and proceed as in Remark~\ref{rmk:angle}.
\end{remark}

\subsection{Intermediate Lemmas}

The following matrix concentration inequality bounds the spectral deviation of $\A$ from its expectation $\Exp[\A]$, and serves as a major step in the proof of \cref{thm:angle} and \cref{lem:eigengap}.

\begin{lemma} \label{lem:concentration}
The noisy signed adjacency matrix $\A$ concentrates around $\Exp[\A]$ in spectral norm with the rate controlled by the number of nodes $n$, edge-noise parameter $p$, and the maximum degree $d_{\max}$.
More formally, we have that
\[\Prob(\norm{\A-\Exp \A}\ge t ) \le 2n e^{\frac{-t^2}{32(1-p)^2d_{\max}}} .\]
\end{lemma}
\begin{proof}[Proof sketch]
We express $\A$ as the sum of independent random matrices and then invoke the matrix Hoeffding inequality~\cite{Tropp_2011}.
\end{proof}

The following result follows on from \cref{lem:concentration} and examines the stability of the spectral structure (dominant eigen gap) of $\A$ under randomness. A nonzero dominant eigen gap is a necessary condition to ensure the uniqueness of $\v_1(\A)$ (up to sign) in \cref{thm:angle} and \cref{thm:proportion}.
Next, we state our lemma.

\begin{lemma} \label{lem:eigengap}
The probability that the dominant eigen gap of $A$ is nonzero is lower bounded by an exponential function dependent on $d_{\max}$, $p$, $n$, and the dominant eigen gap of $\Exp[\A]$.
    \[
        \Prob(\lambda_1(\A)-\lambda_2(\A)>0) \ge 1- 2n e^{\frac{-(\lambda_1(\Exp [\A]) - \lambda_2(\Exp [\A]))^2}{128(1-p)^2d_{\max}}}
    \]
\end{lemma}
\begin{proof}[Proof sketch]
By Weyl's inequality~\cite{Tao_2010} and \cref{lem:concentration}.
\end{proof}

The next lemma examines notable spectral properties of $\Exp[\A]$ and $\E$, used in major steps of the proofs for \cref{thm:angle,thm:proportion}.

\begin{lemma} \label{lem:similar}
The expected signed adjacency matrix $\Exp[\A]$ fulfills:
\begin{align*}
    \Exp[\A] = (1-2p) \, \Diag(\yt) \, \E \, \Diag(\yt) .
\end{align*}
Furthermore, $\Exp[\A]$ is similar to $(1-2p) \, \E$.
That is, $\lambda_i(\Exp[\A]) = (1-2p)\lambda_i(\E)$ for all $i = 1,\dots,n$.
\end{lemma}
\begin{proof}
The generative model presented in \cref{sec:model} can be equivalently described by using the adjacency matrix $\E$ and ``Bernoulli'' variables.
We have $A_{ij} = E_{ij} y_i^* y_j^* z_{ij}$, where $z_{ij}=1$ with probability $1-p$, and $z_{ij}=-1$ with probability $p$.
Note that $\Exp[z_{ij}] = 1-2p$ and therefore $\Exp[A_{ij}] = E_{ij} y_i^* y_j^* \Exp[z_{ij}] = (1-2p) E_{ij} y_i^* y_j^*$.
In matrix form we can write:
\[
\Exp[\A] = (1 - 2p) \left( \E \circ (\yt \yt^\top) \right) .
\]
By using properties of the Hadamard product, this is equivalent to $\Exp[\A] = (1-2p) \, \Diag(\yt) \, \E \, \Diag(\yt)$.
Since $\Diag(\yt) = \inv{\Diag(\yt)}$, $\Exp[\A]$ is similar~\cite{meyer_2000} to $(1-2p) \, \E$ and thus both matrices have the same eigenvalues.
\end{proof}

As discussed before, the next result is pivotal for obtaining \cref{cor:angle,cor:proportion}, which allows us to take advantage of the prior literature regarding the Fiedler value and Cheeger constant.

\begin{lemma} \label{lem:bound_eigengap}
The dominant eigen gap of $\E$ is lower bounded by a linear function of the Fiedler value $\lambda_{n-1}(\L)$, the average degree $d_\avg$, and the maximum degree $d_{\max}$ of the graph.
More formally, we have that
\[\lambda_1(\E) - \lambda_2(\E) \ge \lambda_{n-1}(\L)-d_{\Delta}.\]
\end{lemma}
\begin{proof}[Proof sketch]
By Rayleigh quotients to analyze $\lambda_1(\E)$ and the Courant-Fischer Theorem~\cite{meyer_2000} to analyze $\lambda_2(\E)$.
\end{proof}

\section{Guarantees for Different Classes of Graphs} \label{sec:graphtypes}

Armed with the results in the previous section, and by using Cheeger's inequality~\cite{cheeger1969lower}, we now provide corollaries with explicit guarantees for complete, $d$-regular expanders and graphs with bad expansion properties with few extra Erd\"{o}s–R\'{e}nyi edges.

Next, we show guarantees in angle deviation and proportion of correct labels for complete graphs.
Note that the higher the number of nodes $n$, the higher the probability of success.

\begin{corollary} \label{cor:complete}
    For any $0<\gamma<1$, any $\frac{1}{2}<\alpha<1$ and for a complete graph $G$, the angle deviation fulfills:
    \begin{align*}
    & \Prob\left(|\cos\theta(\v_1(\A), \yt)| \ge 1-\gamma \right) 
    \ge 1-O\left( n \exp\left(\frac{-\max(1,C_{\delta,\gamma})(1-2p)^2n}{(1-p)^2}\right) \right) ,
    \end{align*}
    and the proportion of correct labels fulfills:
    \begin{align*}
    & \Prob(\rho(\v_1(\A), \yt) \ge \alpha) 
    \ge 1-O\left( n \exp\left(\frac{-\max(1,C_{\delta,\alpha})(1-2p)^2n}{(1-p)^2}\right) \right) ,
    \end{align*}
    where $\delta$ and $C_{\delta,\gamma}$ are defined as in \cref{thm:angle}, and $C_{\delta,\alpha}$ is defined as in \cref{thm:proportion}.
\end{corollary}
\begin{proof}
In a complete graph, $d_{\max}=d_{\avg}=n-1$, and Cheeger’s inequality gives 
$\lambda_{n-1}(\L) \geq \frac{h_G^2}{2d_{\max}}$~\cite{cheeger1969lower}.
We also know that $\frac{h_G^2}{d_{\max}} \in \Omega(n)$.
We then invoke \cref{cor:angle,cor:proportion}.
\end{proof}

For completeness, we start by providing some definitions.
A $d$-regular graph is a graph where each node has $d$ neighbors.
A $d$-regular graph with $n$ nodes is an expander with constant $c>0$ if, for every set of nodes $S \subset \{1,...,n\}$ with $|S| \leq n/2$, the number of edges from nodes in $S$ to nodes outside $S$ is greater than or equal to $c d |S|$.

In what follows, we show guarantees for expander graphs.
Again, the higher the number of nodes $n$, the higher the probability of success.

\begin{corollary}
    For any $0<\gamma<1$, any $\frac{1}{2}<\alpha<1$ and for a $d$-regular expander graph $G$ with $d \in \Omega(\log^2 n)$ and constant $c>0$, the angle deviation fulfills:
    \begin{align*}
    & \Prob\left(|\cos\theta(\v_1(\A), \yt)| \ge 1-\gamma \right) 
    \ge 1-O\left( n \exp\left(\frac{-\max(1,C_{\delta,\gamma})(1-2p)^2c^4\log^2 n}{(1-p)^2}\right) \right) ,
    \end{align*}
    and the proportion of correct labels fulfills:
    \begin{align*}
    & \Prob(\rho(\v_1(\A), \yt) \ge \alpha) 
    \ge 1-O\left( n \exp\left(\frac{-\max(1,C_{\delta,\alpha})(1-2p)^2c^4\log^2 n}{(1-p)^2}\right) \right) ,
    \end{align*}
    where $\delta$ and $C_{\delta,\gamma}$ are defined as in \cref{thm:angle}, and $C_{\delta,\alpha}$ is defined as in \cref{thm:proportion}.
\end{corollary}
\begin{proof}
For a $d$-regular expander graph, $d_{\max}=d_\avg=d$.
Cheeger’s inequality gives 
$\lambda_{n-1}(\L) \geq \frac{h_G^2}{2d_{\max}}$~\cite{cheeger1969lower}.
We also know that $\frac{h_G^2}{d_{\max}} \in \Omega(c^4d)$~\cite{hoory2006expander}.
We then invoke \cref{cor:angle,cor:proportion}.
\end{proof} 

The following result holds for any connected graph $G$, even those with bad expansion properties, such as grid graphs, for instance.
We first state a result that appeared in~\cite{krivelevich2015smoothed}, although in a slightly more general form.

\begin{lemma}[Theorem 2 in~\cite{krivelevich2015smoothed}]
\label{lem:cheeger_g_plus_er}
Let $G'= (V,E')$ be a connected graph.
Let $G''= (V,E'')$ be an Erd\"{o}s–R\'{e}nyi graph with edge probability $\epsilon/n$ where $\epsilon \in [1,n]$.
Consider $G = (V, E'\cup E'')$.
We have that $h_G \geq \frac{\epsilon}{256 + 256\log n}$ with probability at least $1 - n^{-2.2 - \frac{\log \epsilon}{2}}$.
\end{lemma}

Finally, we show guarantees for graphs with few extra Erd\"{o}s–R\'{e}nyi edges.
As before, the higher the number of nodes $n$, the higher the probability of success.

\begin{corollary}
Let $G'= (V,E')$ be a connected graph.
Let $G''= (V,E'')$ be an Erd\"{o}s–R\'{e}nyi graph with edge probability $\Omega\left( \frac{\log^6 n}{n} \right)$.
Consider $G = (V, E'\cup E'')$.
For any $0<\gamma<1$, any $\frac{1}{2}<\alpha<1$ and for graph $G$, the angle deviation fulfills:
    \begin{align*}
    & \Prob\left(|\cos\theta(\v_1(\A), \yt)| \ge 1-\gamma \right) 
    \ge 1-O\left( n \exp\left(\frac{-\max(1,C_{\delta,\gamma})(1-2p)^2\log^2 n}{(1-p)^2}\right) \right) ,
    \end{align*}
    and the proportion of correct labels fulfills:
    \begin{align*}
    & \Prob(\rho(\v_1(\A), \yt) \ge \alpha) 
    \ge 1-O\left( n \exp\left(\frac{-\max(1,C_{\delta,\alpha})(1-2p)^2\log^2 n}{(1-p)^2}\right) \right) ,
    \end{align*}
    where $\delta$ and $C_{\delta,\gamma}$ are defined as in \cref{thm:angle}, and $C_{\delta,\alpha}$ is defined as in \cref{thm:proportion}.
\end{corollary}
\begin{proof}
For graph $G$, $d_{\max} \in \Omega(\epsilon)$ and $d_\avg \in \Omega(\epsilon)$.
Cheeger’s inequality gives 
$\lambda_{n-1}(\L) \geq \frac{h_G^2}{2d_{\max}}$~\cite{cheeger1969lower}.
By \cref{lem:cheeger_g_plus_er} we have that $h_G \in \Omega\left( \frac{\epsilon}{\log n} \right)$.
We then invoke \cref{cor:angle,cor:proportion}.
\end{proof}

\section{Experimental Validation} \label{sec:exp}

\begin{figure}
    \centering
    \begin{subfigure}[t]{0.48\textwidth}
        \centering
        \includegraphics[width=\linewidth]{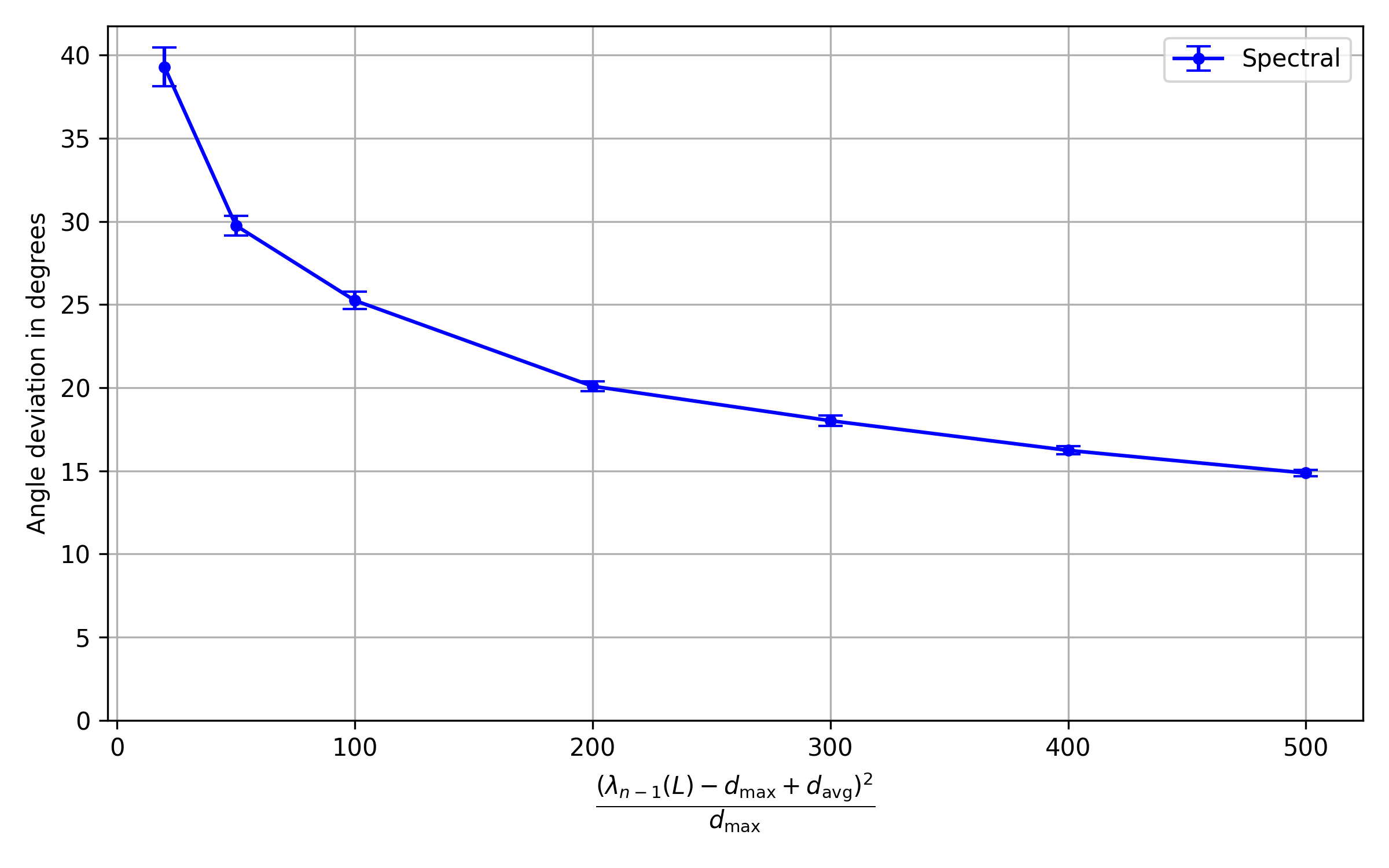}
         \caption{}
    \end{subfigure}
    \hfill
    \begin{subfigure}[t]{0.48\textwidth}
        \centering
        \includegraphics[width=\linewidth]{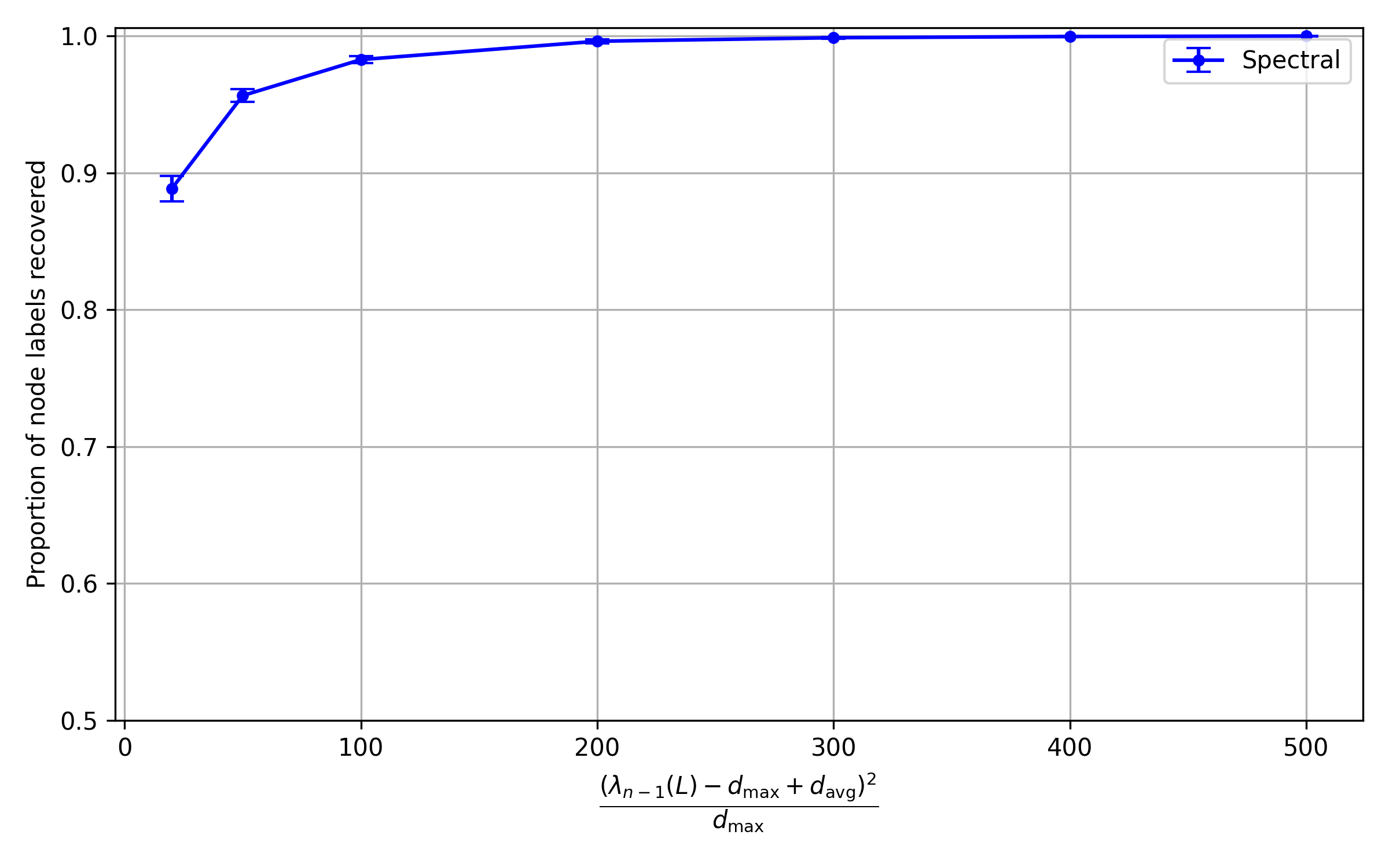}
        \caption{} 
    \end{subfigure}
    \caption{Effect of the structural term $\frac{(\lambda_{n-1}(\L)-d_{\max}+d_\avg)^2}{d_{\max}}$ on the angle deviation (a) and proportion of correctly recovered node labels (b) for Erd\"{o}s–R\'{e}nyi graphs (with varying edge density and $n$) with $p=0.4$.
    Error bars at 95\% confidence level over 30 repetitions.
    As predicted by \cref{cor:angle,cor:proportion}, the higher the structural term, the lower the angle deviation and the higher the proportion of correctly recovered node labels.}
    \label{fig:q-ER}
\end{figure}

\begin{figure}
    \centering
    \includegraphics[width=0.48\linewidth,trim={10 0 50 35},clip]{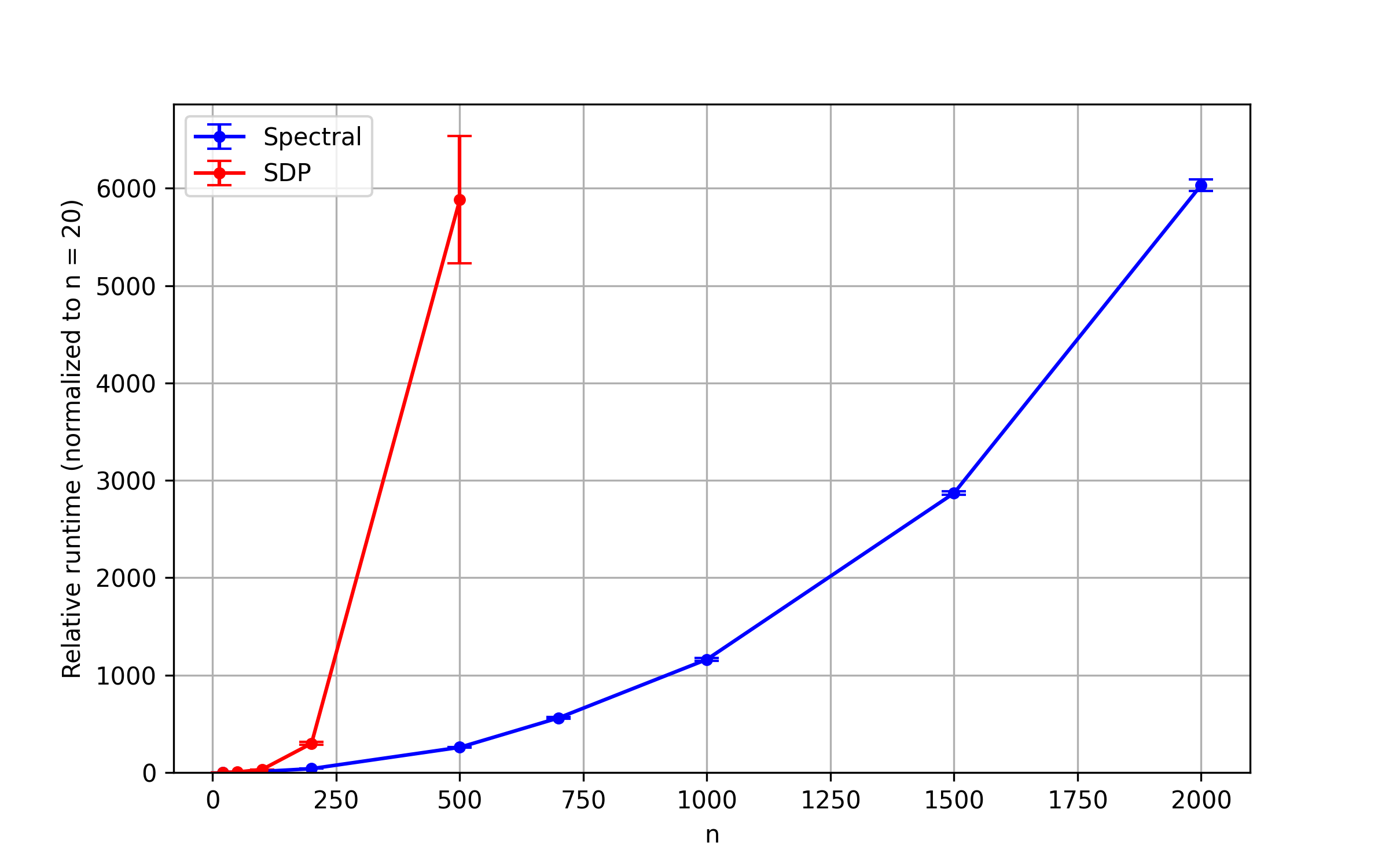}
    \caption{Effect of number of nodes $n$ on the runtime of our spectral method (Spectral) and semidefinite programming (SDP) for Erd\"{o}s–R\'{e}nyi graphs (60\% edge density) for $p=0.4$.
    Error bars at 95\% confidence level over 10 repetitions.
    Spectral structured prediction is a far more efficient algorithm.
    SDP was run only up to graphs of size 500 nodes due its high runtime.}
    \label{fig:runtime-comparison}
\end{figure}

\begin{figure}
    \centering
    \begin{subfigure}[t]{0.48\textwidth}
        \centering
        \includegraphics[width=\linewidth,trim={10 0 50 35},clip]{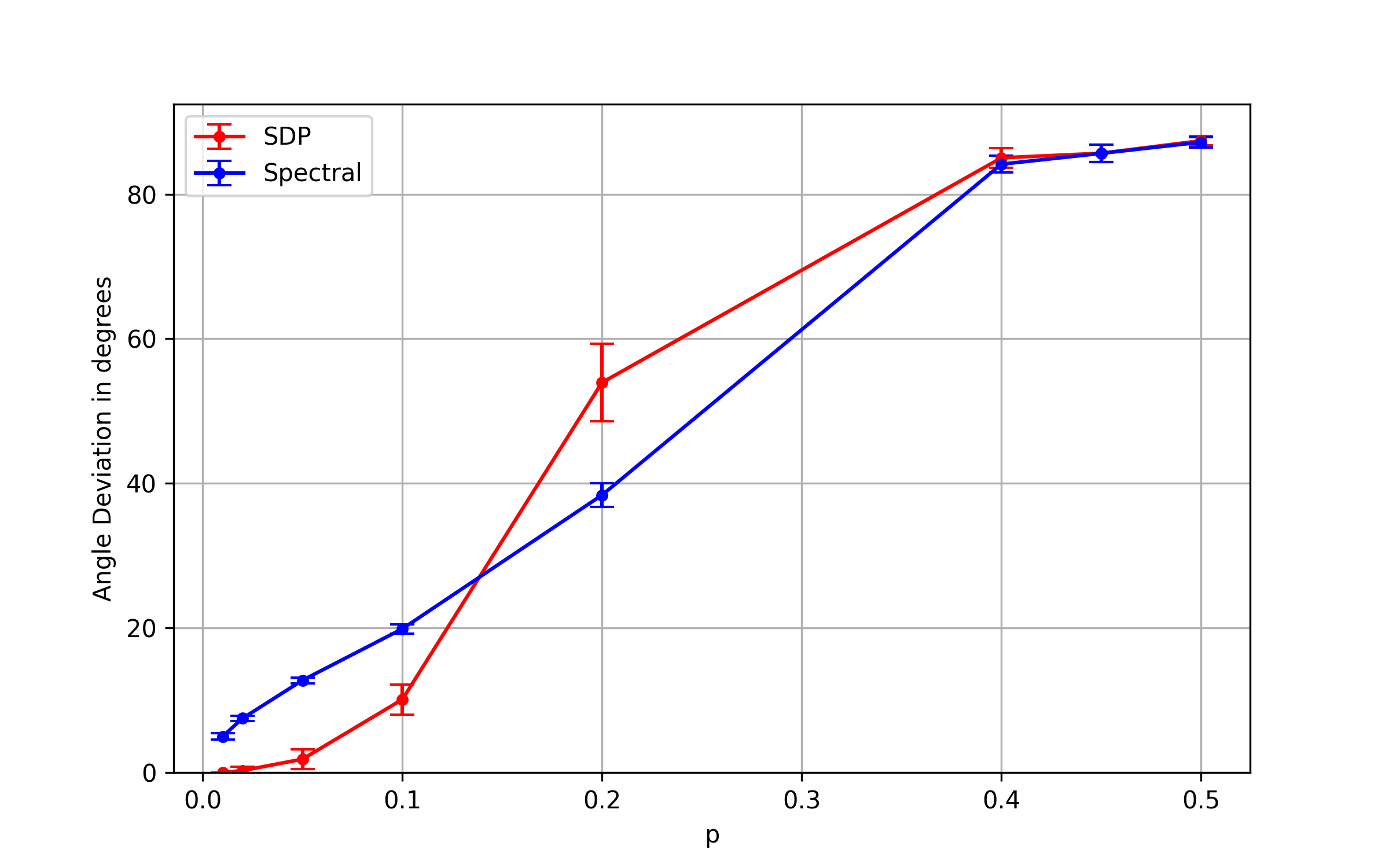}
        \caption{}
    \end{subfigure}
    \hfill
    \begin{subfigure}[t]{0.48\textwidth}
        \centering
        \includegraphics[width=\linewidth,trim={10 0 50 35},clip]{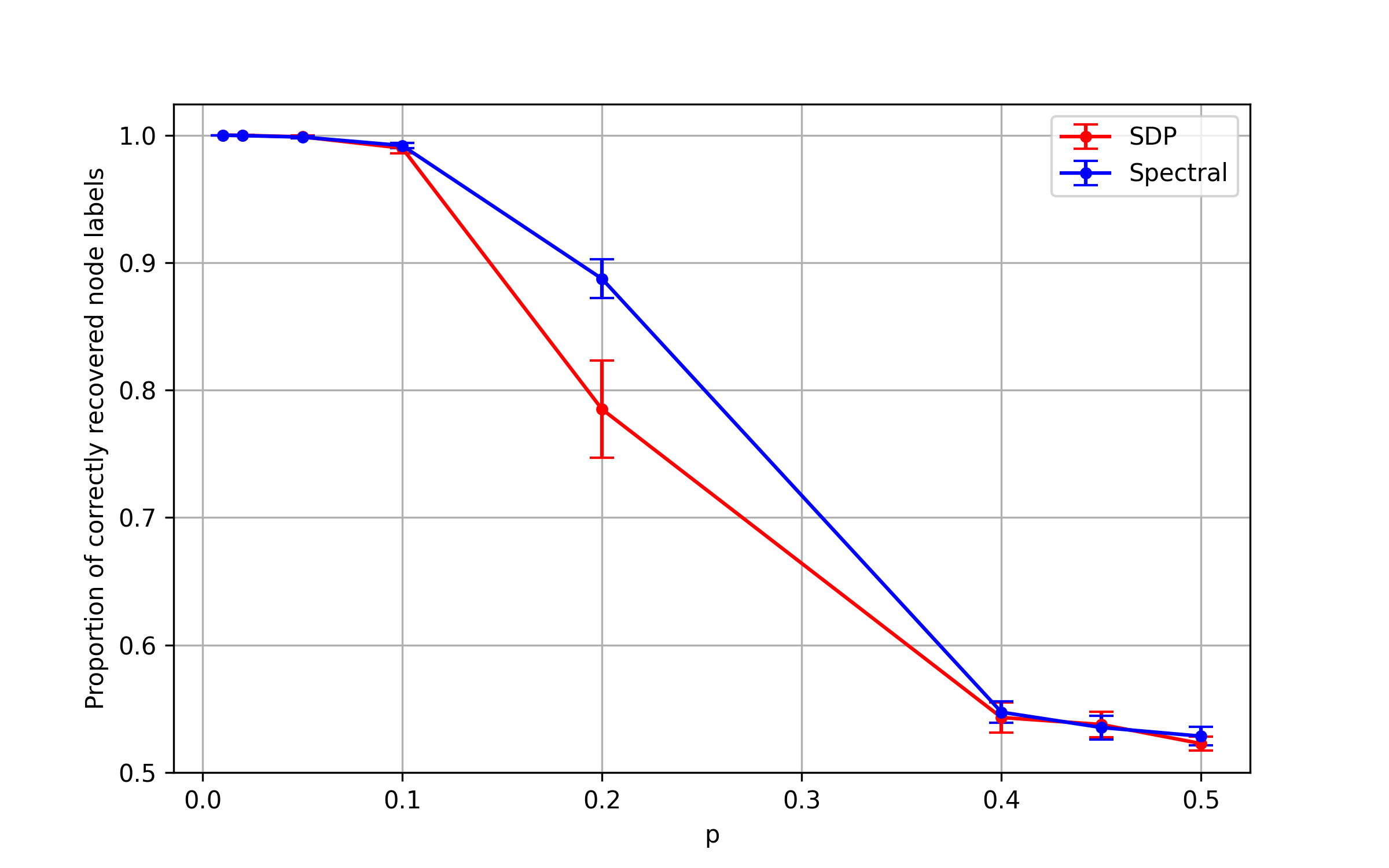}
        \caption{}
    \end{subfigure}
    \caption{Effect of edge-noise parameter $p$ on the angle deviation (a) and proportion of correctly recovered node labels (b) for $d$-regular expanders ($d=6$) of $n=200$ nodes.
    Error bars at 95\% confidence level over 30 repetitions.
    Our spectral method (Spectral) and semidefinite programming (SDP) produced comparable results.}
    \label{fig:expander-p}
\end{figure}

We validate our theory and confirm the benefits of spectral structured prediction through synthetic experiments.
We conducted these experiments through various parameters of interest, such as number of nodes $n$, edge-noise probability $p$, and the structural term $\frac{\left(\lambda_{n-1}(\L)-d_{\max}+d_\avg\right)^2}{d_{\max}}$ involved in the theoretical results in \cref{cor:angle,cor:proportion}.
We generated different classes of graphs, including Erd\"{o}s–R\'{e}nyi graphs, regular expander graphs and complete graphs.

\paragraph{Validation of Our Theory.}

In \cref{fig:q-ER}, we show that the structural term $\frac{\left(\lambda_{n-1}(\L)-d_{\max}+d_\avg\right)^2}{d_{\max}}$ indeed drives the accuracy of structured prediction.
As predicted by \cref{cor:angle}, the higher the structural term, the lower the angle deviation.
As predicted by \cref{cor:proportion}, the higher the structural term, the higher the proportion of correctly recovered node labels.

\paragraph{Runtime Comparison.}

Next, we tested how the number of nodes $n$ affect the runtime of our spectral method, and that of semidefinite programming.
We chose semidefinite programming since this method has theoretical guarantees that we discuss in \cref{rmk:angle,rmk:proportion}.
More specifically, semidefinite programming was analyzed in~\cite{bello_2019,ke_2025}.
Unfortunately, the methods of~\cite{globerson_2015,foster2018inference} apply only to specific classes of graphs.
We normalize each method (our spectral method and semidefinite programming) independently to the runtime for $n=20$, to account for differences in implementations.
That is, both methods have runtime $1$ for $n=20$.
\cref{fig:runtime-comparison} shows that our spectral method runs significantly faster than semidefinite programming.

\paragraph{Comparison Across Different Classes of Graphs.}

Given that our spectral method is considerably faster, we then ask whether it produces comparable results to those of semidefinite programming, regarding the effect of the edge-noise parameter $p$ in the angle deviation and proportion of correctly recovered node labels.
In order to do this, we run experiments on different classes of graphs.
\cref{fig:expander-p} shows that our spectral method and semidefinite programming produce comparable results for regular expander graphs.
(Erd\"{o}s–R\'{e}nyi graphs and complete graphs shown in Appendix~\ref{app:moregraphs}.)

In Appendix~\ref{app:moregraphs}, we also show additional experiments regarding the effect of the number of nodes $n$ in the angle deviation and proportion of correctly recovered node labels for Erd\"{o}s–R\'{e}nyi graphs, regular expanders graphs, as well as for complete graphs.

Notably, while the dominant eigenvector is a simpler relaxation of \cref{eq:pairwise} as compared to SDP, our empirical results demonstrate that our method may actually recover node labels better than SDP in high noise (high $p$) regimes since SDP may overfit upon the noise as it is a tighter relaxation of \cref{eq:pairwise}.

\paragraph{Real-World Experiments.}

In Appendix~\ref{app:realworld}, we include an experiment on a large real-world case with 131,828 nodes and 841,372 edges.
Our spectral algorithm runs in 0.1484 seconds while SDP would struggle at this scale in commercial solvers.

\section{Concluding Remarks}

There are parallels with our spectral approach with spectral clustering within the context of stochastic block models~\cite{Lei_2015}, \cite{abbe_2016}. However, the fundamental model that we consider is different to stochastic block models, and existing guarantees rely on $k$-means, a problem that is NP-hard~\cite{aloise_2009} and NP-hard to approximate in polynomial time~\cite{awasthi_2015}.
In contrast, our spectral method is polynomial time.

There are several ways of extending our results.
A first compelling avenue is to analyze social network models, such as stochastic block models~\cite{abbe_2016,wang_2020} in which the noisy adjacency matrix is nonnegative.
Currently, our probabilistic guarantees apply to only $k=2$ node labels.
In order to extend this model to account for $k>2$ node labels, we would need to analyze the $k-1$ leading eigenvectors of $\A$, similar in spirit to prior work in social networks~\cite{Lei_2015}.
While we currently focus on pairwise interactions (i.e., edges), we could also extend our analysis to more $m$-order interactions for $m>2$ (i.e., hyperedges) as the problem presented in~\cite{benson_tensor_2015}.
This would lead to structured predictions for $m$-order tensors, though with possibly much greater difficulty.

\bibliographystyle{plain}
\bibliography{references}

\clearpage
\appendix

\section{Technical Lemmas} \label{app:techlemmas}

The following lemmas on matrix concentration theory and eigenvector perturbation play important roles in the proofs of our main theorems. 

\begin{lemma}[Matrix Hoeffding inequality, Theorem 1.3~\cite{Tropp_2011}] \label{lem:hoeffding} 
For a finite sequence $\{\X_k\}$ of independent, random, self-adjoint matrices with dimension $d$, and let $\{\M_k\}$ be a sequence of fixed self-adjoint matrices. Assuming that each matrix satisfies $\Exp[\X_k]=\zero\zero^\top$ and 
$\X_k^2 \preceq \M_k^2$ almost surely. Then for all $t \ge 0$, 
\[
    \Prob\left(\lambda_1\left(\sum_k \X_k\right) \geq t\right) \leq d e^{-t^2 / 8 \sigma^2},
\]
where $\sigma^2=\norm{\sum_k \M_k^2}$.
\end{lemma}

\begin{lemma}[Davis-Kahan $\sin \theta$ Theorem~\cite{Kahan_1970}] \label{lem:kahan} 
Let $\M, \F$ be symmetric matrices with
eigenvectors $\v_1(\M)$ and $\v_1(\F)$ corresponding to their dominant eigenvalues 
$\lambda_1(\M)$ and $\lambda_1(\F)$.
If $\M$ has a nonzero eigen gap, i.e., $\lambda_1(\M) - \lambda_2(\M) > 0$, then the angle $\theta(\v_1(\M), \v_1(\F))$ satisfies
\[
\sin \theta(\v_1(\M), \v_1(\F)) \;\leq\; \frac{\| \M - \F \|}{\lambda_1(\M) - \lambda_2(\M)} .
\]
\end{lemma}

\section{Detailed Proofs} \label{app:proofs}

In this section, we provide detailed proofs of our theorems and lemmas in the main text.

\subsection{Proof of \cref{thm:angle}}

\begin{proof}
We proceed as follows:
We split $|\cos\theta(\v_1(\A), \yt)|$ into two terms: one random variable $|\cos\theta(\v_1(\A), \v_1(\Exp \A))|$ and one constant $|\cos\theta(\v_1(\Exp \A), \yt)|$.
The two terms are then properly bounded and then used to analyze $|\cos\theta(\v_1(\A), \yt)|$.
We finalize by ensuring the uniqueness of $\v_1(\A)$ (up to sign).

\paragraph{Step 1: Bounding the random variable $|\cos\theta(\v_1(\A), \v_1(\Exp \A))|$.}

By invoking \cref{lem:kahan} with $\F=\A$ and $\M=\Exp \A$, we have: 
\[
    |\sin\theta(\v_1(\A),\v_1(\Exp \A))| \le \frac{\norm{\A-\Exp \A}}{\lambda_1(\Exp \A)-\lambda_2(\Exp \A)}
\]
Note that by \cref{lem:similar}, we have $\lambda_1(\Exp \A)-\lambda_2(\Exp \A) = (1-2p)(\lambda_1(\E)-\lambda_2(\E)) >0$ since $\lambda_1(\E)$ has multiplicity 1~\cite{meyer_2000}. 
By trigonometric identities, we have:
\begin{align*}
|\cos\theta(\v_1(\A),\v_1(\Exp \A))| 
&= \sqrt{1 - \sin^2(\theta(\v_1(\A),\v_1(\Exp \A)))} \\
&\ge 1 - \sin^2(\theta(\v_1(\A),\v_1(\Exp \A))) \\
&\ge 1 - \frac{\norm{\A-\Exp \A}^2}{\left( \lambda_1(\Exp \A) - \lambda_2(\Exp \A) \right)^2}
\end{align*}

From the above and by \cref{lem:concentration}, for any $0<\epsilon<1$ we have:
\begin{align*}
\Prob(|\cos\theta(\v_1(\A),\v_1(\Exp \A))| \ge 1-\epsilon)
&\ge \Prob\left(\epsilon \ge \frac{\norm{\A-\Exp \A}^2}{\left( \lambda_1(\Exp \A) - \lambda_2(\Exp \A) \right)^2}\right) \\
&= \Prob(\norm{\A-\Exp \A} \le \sqrt{\epsilon} \left( \lambda_1(\Exp \A) - \lambda_2(\Exp \A) \right)) \\
&= 1- \Prob(\norm{\A-\Exp \A} \ge \sqrt{\epsilon} \left( \lambda_1(\Exp \A) - \lambda_2(\Exp \A) \right)) \\
&\ge 1-2n \exp\left( \frac{-\epsilon (\lambda_1(\Exp \A)-\lambda_2(\Exp \A))^2}{32(1-p)^2d_{\max}} \right).
\end{align*}

\paragraph{Step 2: Bounding the constant $|\cos\theta(\v_1(\Exp \A), \yt)|$.}

By \cref{lem:similar} and since $\Diag(\yt)^2$ is the identity matrix, we have:
\begin{align*}
\Exp[\A] \, \Diag(\yt) \v_1(\E)
&= (1 - 2p) \, \Diag(\yt) \, \E \, \Diag(\yt)^2 \v_1(\E) \\
&= (1 - 2p) \, \Diag(\yt) \, \E \, \v_1(\E) \\
&= (1 - 2p) \, \lambda_1(\E) \Diag(\yt) \v_1(\E) ,
\end{align*}
where the last step follows from the definition of eigenvalues, i.e., $\E\,\v_1(\E) = \lambda_1(\E) \v_1(\E)$.
Note that the above implies that $\lambda_1(\Exp \A) = (1-2p)\lambda_1(\E)$ and $\v_1(\Exp \A) = \Diag(\yt) \v_1(\E)$.
This can be confirmed by the eigenvalue definition, i.e. $\Exp[\A]\v_1(\Exp \A) = \lambda_1(\Exp \A)\v_1(\Exp \A)$ which leads to the same expression as above.
Therefore, since $\yt \in \{-1,+1\}^n$, we have:
\begin{align*}
    |\cos\theta(\v_1(\Exp \A),\yt))|
    &= |\cos\theta(\Diag(\yt) \v_1(\E),\yt)| \\
    &= \frac{|\innerprod{ \Diag(\yt) \v_1(\E) }{ \yt }|}{\norm{\Diag(\yt) \v_1(\E)}\norm{\yt}} \\
    &= \frac{|\sum_{i} y_i^* [\v_1(\E) ]_i y_i^*|}{\norm{\v_1(\E)}\norm{\one}} \\
    &=  \frac{|\innerprod{\v_1(\E)}{\one}|}{\norm{\v_1(\E)}\norm{\one}} \\
    &= |\cos\theta(\v_1(\E), \one)| \\
    &= \delta,
\end{align*}
where the last step follows from our definition in the theorem statement.

\paragraph{Step 3: Using the random variable $|\cos\theta(\v_1(\A), \v_1(\Exp \A))|$ and the constant $|\cos\theta(\v_1(\Exp \A), \yt)|$ to bound $|\cos\theta(\v_1(\A), \yt)|$.}

For simplicity of presentation, in what follows, we argue for angles $\theta(\v_1(\A), \yt) \in [0,\pi/2]$ and $\theta(\v_1(\A), \v_1(\Exp \A)) \in [0,\pi/2]$.\footnote{Arguing for an angle $\theta \in [\pi/2,\pi]$ requires considering $\theta-\pi/2$, and arguing for an angle $\theta \in [-\pi,0]$ requires considering $-\theta$ instead.}
Let $\theta_1 = \theta(\v_1(\A), \v_1(\Exp \A))$ and $\theta_2 = \theta(\v_1(\Exp \A), \yt)$.
By using trigonometric identities and the result in Step 2, we have:
\begin{align*}
|\cos\theta(\v_1(\A), \yt)| &= \cos\theta(\v_1(\A), \yt) \\
&\ge \cos(\theta_1 + \theta_2) \\
&= \cos\theta_1 \cos\theta_2 - \sin\theta_1 \sin\theta_2 \\
&\ge \delta \cos\theta_1 - \sqrt{1-\delta^2} \sin\theta_1 .
\end{align*}
Since we are arguing for an angle $\theta_1 \in [0,\pi/2]$ for simplicity of presentation, by the result in Step 1, we have:
\begin{align*}
    \Prob(|\cos\theta(\v_1(\A), \yt)| \ge 1-\gamma) 
    &\ge \Prob(\delta \cos\theta_1 - \sqrt{1-\delta^2} \sin\theta_1 \ge 1-\gamma) \\
    &= \Prob(\cos{\theta_1}\le \delta - \delta \gamma - \sqrt{\gamma (\delta^2-1)(\gamma-1)}) + \Prob(\cos{\theta_1}\ge \delta - \delta \gamma + \sqrt{\gamma (\delta^2-1)(\gamma-1)}) \\
    &\ge \Prob(\cos{\theta_1}\ge \delta-\delta \gamma + \sqrt{\gamma (\delta^2-1)(\gamma-1)}) \\
    &= \Prob(\cos\theta(\v_1(\A), \v_1(\Exp \A)) \ge \delta-\delta \gamma + \sqrt{\gamma (\delta^2-1)(\gamma-1)}) \\
    &= \Prob(|\cos\theta(\v_1(\A), \v_1(\Exp \A))| \ge \delta-\delta \gamma + \sqrt{\gamma (\delta^2-1)(\gamma-1)}) \\
    &\ge 1-2n \exp\left( \frac{-(1-\delta + \delta \gamma - \sqrt{\gamma (\delta^2-1)(\gamma-1)}) (\lambda_1(\Exp \A)-\lambda_2(\Exp \A))^2}{32(1-p)^2d_{\max}} \right) \\
    &= 1-2n \exp\left( \frac{-(1-\delta + \delta \gamma - \sqrt{\gamma (\delta^2-1)(\gamma-1)}) (1-2p)^2(\lambda_1(\E)-\lambda_2(\E))^2}{32(1-p)^2d_{\max}} \right),
\end{align*}
where the last simplification step follows from \cref{lem:similar}.

\paragraph{Step 4: Ensuring the uniqueness of $\v_1(\A)$ (up to sign).}

In all the expressions above, we have assumed that $\v_1(\A)$ is unique (up to sign).
As a counterexample, when the $\lambda_1(\A)=\lambda_2(\A)$, there exists infinitely many dominant eigenvectors.
In order to ensure uniqueness of $\v_1(\A)$ (up to sign), we need to ensure that $\lambda_1(\A)-\lambda_2(\A)>0$.
By \cref{lem:eigengap}, we have:
\begin{align*}
\Prob\left(\lambda_1(\A)-\lambda_2(\A)>0 \right) & \ge 1 - 2n \exp\left( \frac{ -(\lambda_1(\Exp \A)-\lambda_2(\Exp \A))^2 }{ 128(1-p)^2 d_{\max} } \right) \\
&= 1 - 2n \exp\left( \frac{ -(1-2p)^2(\lambda_1(\E)-\lambda_2(\E))^2 }{ 128(1-p)^2 d_{\max} } \right)
\end{align*} 
where the last simplification step follows from \cref{lem:similar}.
The final bound comes from joining the results in Steps 3 and 4, and since $(1-r)(1-s) \geq 1-r-s$ for $r,s\in[0,1]$.
\end{proof}

\subsection{Proof of \cref{thm:proportion}}

\begin{proof}
Next, we reason about the maximum angle deviation in order to guarantee a minimum proportion of correctly recovered node labels $\alpha$.
Recall $\yt=(y_1^*,y_2^*,\ldots,y_n^*)$ where each $y_i^* \in \{-1,+1\}$.
Note that the definition of $\rho(\v_1(\A),\yt)$ considers both $\v_1(\A)$ and $-\v_1(\A)$ and thus $\rho(\v_1(\A),\yt) \geq 1/2$.
Therefore, without loss of generality, assume $\alpha>1/2$. Let $\alpha=k/n$ where $k>n/2$.

Let $\S$ be the set of vectors with $l$ sign agreements with $\yt$, for $n/2\leq l<k$, defined as \begin{align*}
\S = \bigcup_{T \subseteq \{1,\dots,n\}, n/2 \leq |T| < k} \S_T \, ,
\end{align*}
where
\begin{align*}
\S_T = \{\s \in \mathbb R^n : s_i = (2 \times 1[i \in T]-1)\,\sign(y_i^*)\,c_i, c_i>0, \norm{\c}_2 = 1\} \, .
\end{align*}
Consider $\v_1(A) \in \S$.
For $n/2\leq l<k$, without loss of generality, assume the first $l$ entries agree in sign with $\yt$.
That is, $\v_1(\A)=(\sign(y_1^*)\,c_1,\dots,\sign(y_l^*)\,c_l,
-\sign(y_{l+1}^*)\,c_{l+1},\dots,-\sign(y_n^*)\,c_n)$.
Recall $\norm{\v_1(\A)}_2=1$ and $\norm{\yt}_2=\sqrt{n}$.
For clarity, let $\one_m$ be a vector of $m$ ones.
We have:
\begin{align*}
|\cos\theta(\v_1(\A),\yt)| & = \frac{|\innerprod{\v_1(\A)}{\yt}|}{\norm{\v_1(\A)}_2\norm{\yt}_2} \\
 & = |c_1+\dots+c_l-c_{l+1}-\dots-c_n|/\sqrt{n} \\
 & \leq \max(c_1+\dots+c_l,c_{l+1}+\dots+c_n)/\sqrt{n} \\
 & = \max(\innerprod{(c_1,\dots,c_l)}{\one_l},\innerprod{(c_{l+1},\dots,c_n)}{\one_{n-l}}/\sqrt{n} \\
 & \leq \max(\sqrt{c_1^2+\dots+c_l^2}\sqrt{l},\sqrt{c_{l+1}^2+\dots+c_n^2}\sqrt{n-l})/\sqrt{n} \quad \text{by Cauchy-Schwarz inequality} \\
 & \leq \max(\sqrt{l},\sqrt{n-l})/\sqrt{n} \\
 & = \sqrt{l/n} \\
 & < \sqrt{k/n} \\
 & = \sqrt{\alpha}.
\end{align*}
Thus, if $\rho(\v_1(\A),\yt)<\alpha$ then $|\cos\theta(\v_1(\A),\yt)|<\sqrt{\alpha}$. By the contrapositive: if $|\cos\theta(\v_1(\A),\yt)|\geq\sqrt{\alpha}$ then $\rho(\v_1(\A),\yt)\geq\alpha$.
Therefore:
\begin{align*}
\Prob(\rho(\v_1(\A), \yt) \ge \alpha) \geq \Prob\left( |\cos\theta(\v_1(\A), \yt)| \ge \sqrt{\alpha} \right) .
\end{align*}
By invoking \cref{thm:angle} with $\gamma = 1-\sqrt{\alpha}$, we prove our claim.
\end{proof}

\subsection{Proof of \cref{lem:concentration}}

\begin{proof}
Note that:
\begin{align*}
\norm{\A-\Exp \A} & = \max(\lambda_1(\A-\Exp \A), \; -\lambda_n(\A-\Exp \A)) \\
 & = \max(\lambda_1(\A-\Exp \A), \; \lambda_1(\Exp[\A]-\A)).
\end{align*}
Since $\norm{\A-\Exp \A}$ depends on both $\lambda_1(\A-\Exp \A)$ and $\lambda_1(\Exp[\A]-\A)$, we bound them separately. 

To bound $\Prob(\lambda_1(\A-\Exp \A) \ge t)$, we apply \cref{lem:hoeffding} on $\A-\Exp[\A]$. 
Let $(i,j) \in E$ be the $k$-th edge in the edge set $E$.
We define:
\[\X_k = (A_{ij} - \Exp[A_{ij}]) (\e_i \e_j^\top + \e_j \e_i^\top)\] 

Since each $A_{ij}$ is independent, the above satisfies the requirement in \cref{lem:hoeffding} that $\X_k$ are independent, random, self-adjoint matrices.
Since $A_{ij} - \Exp[A_{ij}]$ has expected value zero, then $\Exp \X_k = \zero\zero^\top$.
Furthermore, note that $(A_{ij} - \Exp[A_{ij}])^2 \le (-1-(1-2p))^2 = 4(1-p)^2$.
Then, we let $\M_k^2 = 4(1-p)^2 (\e_i \e_i^\top + \e_j \e_j^\top)$ in order to satisfy $\X_k^2 \preceq \M_k^2$, also required in \cref{lem:hoeffding}.
Note that:
\begin{align*}
\sigma^2 & = \norm{\sum_k \M_k^2} \\
 & = \norm{4(1-p)^2 \sum_{k=1}^{n} d_i(\e_i \e_i^\top)} \\
 & = 4(1-p)^2 d_{\max}.
\end{align*}
Thus, by invoking \cref{lem:hoeffding}, we have:
\begin{align*}
\Prob(\lambda_1(\A-\Exp \A) \ge t) & = \Prob\left( \lambda_1\left(\sum_k \X_k\right) \ge t \right) \\
 & \le n e^{-t^2/8\sigma^2} \\
 & = n e^{\frac{-t^2}{32(1-p)^2 d_{\max}}}.
\end{align*}

To bound $\Prob(\lambda_1(\Exp[\A]-\A)) \ge t)$, we apply \cref{lem:hoeffding} on $\Exp[\A]-\A$ and undergo nearly the same process as above, but now we negate each of $\X_k$ from before, while $\M_k$ and $\sigma^2$ remains the same.
Consequently, we get the following upper bound for $\Prob(\lambda_1(\Exp[\A]-\A))$ as the second bound needed to complete our proof.
Note that this turns out to be the same bound as before.
That is,
\[
\Prob(\lambda_1(\Exp[\A]-\A) \ge t) \le n e^{\frac{-t^2}{32(1-p)^2 d_{\max}}}.
\]

By using the union bound, we have:
\begin{align*}
\Prob\left(\norm{\A-\Exp \A} \ge t \right) & = \Prob\left( \max(\lambda_1(\A-\Exp \A), \; \lambda_1(\Exp[\A]-\A)) \geq t \right) \\
 & = \Prob\left( \lambda_1(\A-\Exp \A) \geq t \lor \lambda_1(\Exp[\A]-\A) \geq t \right) \\
 & \le \Prob\left( \lambda_1(\A-\Exp \A) \ge t \right) + \Prob\left( \lambda_1(\Exp[\A] - \A) \ge t \right) \\
 & \le 2 n e^{\frac{-t^2}{32(1-p)^2 d_{\max}}},
\end{align*}
and we prove our claim.
\end{proof}

\subsection{Proof of \cref{lem:eigengap}}

\begin{proof}
By the eigenvalue stability inequality, based on Weyl's inequality~\cite{Tao_2010}, we have that:
\[-\norm{\A-\Exp \A} \le \lambda_1(\A)-\lambda_1(\Exp \A) \le \norm{\A-\Exp \A},\] and
\[-\norm{\A-\Exp \A} \le \lambda_2(\A)-\lambda_2(\Exp \A) \le \norm{\A-\Exp \A}.\]

Putting these together, we have:
\[\lambda_1(\A)-\lambda_2(\A) \ge \lambda_1(\Exp \A) - \lambda_2(\Exp \A) -2\norm{\A-\Exp \A}.\]

By invoking \cref{lem:concentration} with $t=\frac{\lambda_1(\Exp \A) - \lambda_2(\Exp \A)}{2}$, we have: 
\[
\Prob\left( \norm{\A-\Exp \A} \ge \frac{\lambda_1(\Exp \A) - \lambda_2(\Exp \A)}{2} \right) \leq 2n e^{\frac{-(\lambda_1(\Exp \A) - \lambda_2(\Exp \A))^2}{128(1-p)^2d_{\max}}}.
\]

Finally, note that:
\begin{align*}
\Prob(\lambda_1(\A)-\lambda_2(\A)>0) 
&\ge \Prob(\lambda_1(\Exp \A) - \lambda_2(\Exp \A)-2\norm{\A-\Exp \A}>0) \\
&= \Prob(\norm{\A-\Exp \A} < \frac{\lambda_1(\Exp \A)-\lambda_2(\Exp \A)}{2}) \\
&= 1 - \Prob(\norm{\A-\Exp \A} \ge \frac{\lambda_1(\Exp \A)-\lambda_2(\Exp \A)}{2}) \\
&\ge 1- 2n e^{\frac{-(\lambda_1(\Exp \A) - \lambda_2(\Exp \A))^2}{128(1-p)^2d_{\max}}},
\end{align*}
which proves our claim.
\end{proof}

\subsection{Proof of \cref{lem:bound_eigengap}}

\begin{proof}
We first find a lower bound for $\lambda_1(\E)$, by using its Rayleigh Quotient definition:
\begin{align*}
\lambda_1(\E) & = \max_{\x \neq \zero} \frac{\x^\top \E \x}{\x^\top \x} \\
 & \ge \frac{\one^\top \E \one}{\one^\top \one} \\
 & = \frac{\sum_{i=1}^n d_i}{n} \\
 & = d_\avg .
\end{align*}
Then, we find an upper bound for $\lambda_2(\E)$, by using the Courant-Fischer Theorem~\cite{meyer_2000}, as follows:
\begin{align*}
\lambda_2(\E) & = \min_{\b} \max_{\norm{\x}=1, \b^\top \x=0} \innerprod{\E\x}{\x} \\
&\le \max_{\norm{\x}=1, \one^\top \x=0} \x^\top \E \x \\
&=   \max_{\norm{\x}=1, \one^\top \x=0} \left(\x^\top \D \x - \x^\top \L \x \right) \\
&= \max_{\norm{\x}=1, \one^\top \x=0} \left(\sum_{i=1}^n d_i x_i^2 - \x^\top \L \x \right) \\
&= \max_{\norm{\x}=1, \one^\top \x=0} \left(\sum_{i=1}^n d_i x_i^2 \right) - \min_{\norm{\x}=1, \one^\top \x=0} \x^\top \L \ x \\
&= d_{\max} - \lambda_{n-1}(\L).
\end{align*}
By combining both bounds, we get $\lambda_1(\E) - \lambda_2(\E) \ge \lambda_{n-1}(\L)-d_{\max}+d_\avg$, and we prove our claim.
\end{proof}

\section{Additional Experiments}

\subsection{More Comparisons Across Different Classes of Graphs} \label{app:moregraphs}

\begin{figure}
    \centering
    \begin{subfigure}[t]{0.48\textwidth}
        \centering
        \includegraphics[width=\linewidth,trim={10 0 50 35},clip]{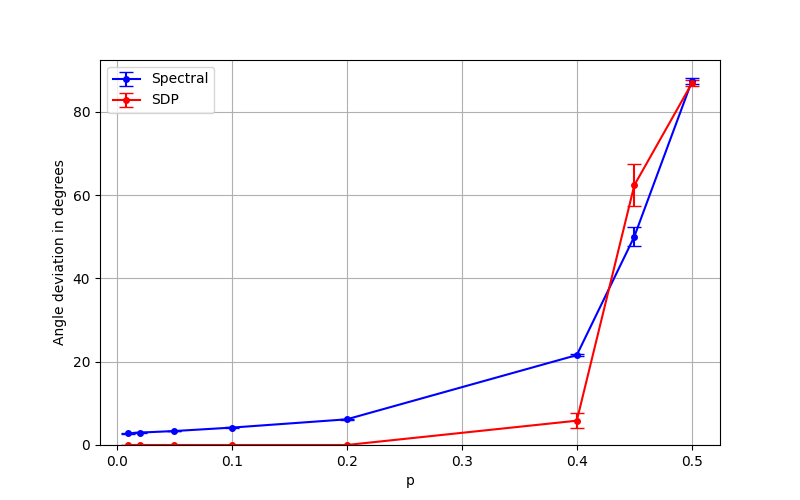}
        \caption{}
    \end{subfigure}
    \hfill 
    \begin{subfigure}[t]{0.48\textwidth}
        \centering
        \includegraphics[width=\linewidth,trim={10 0 50 35},clip]{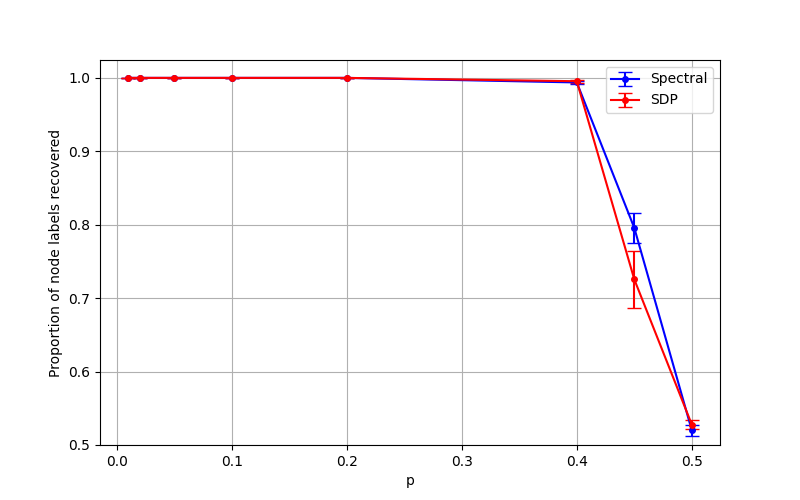}
        \caption{}
    \end{subfigure}
    \caption{Effect of edge-noise parameter $p$ on the angle deviation (a) and proportion of correctly recovered node labels (b) for Erd\"{o}s–R\'{e}nyi graphs (60\% edge density) of $n=200$ nodes.
    Error bars at 95\% confidence level over 30 repetitions.
    Our spectral method (Spectral) and semidefinite programming (SDP) produced comparable results.}
    \label{fig:ER-p}
\end{figure}

\begin{figure}
    \centering
    \begin{subfigure}[t]{0.48\textwidth}
        \centering
        \includegraphics[width=\linewidth,trim={10 0 50 35},clip]{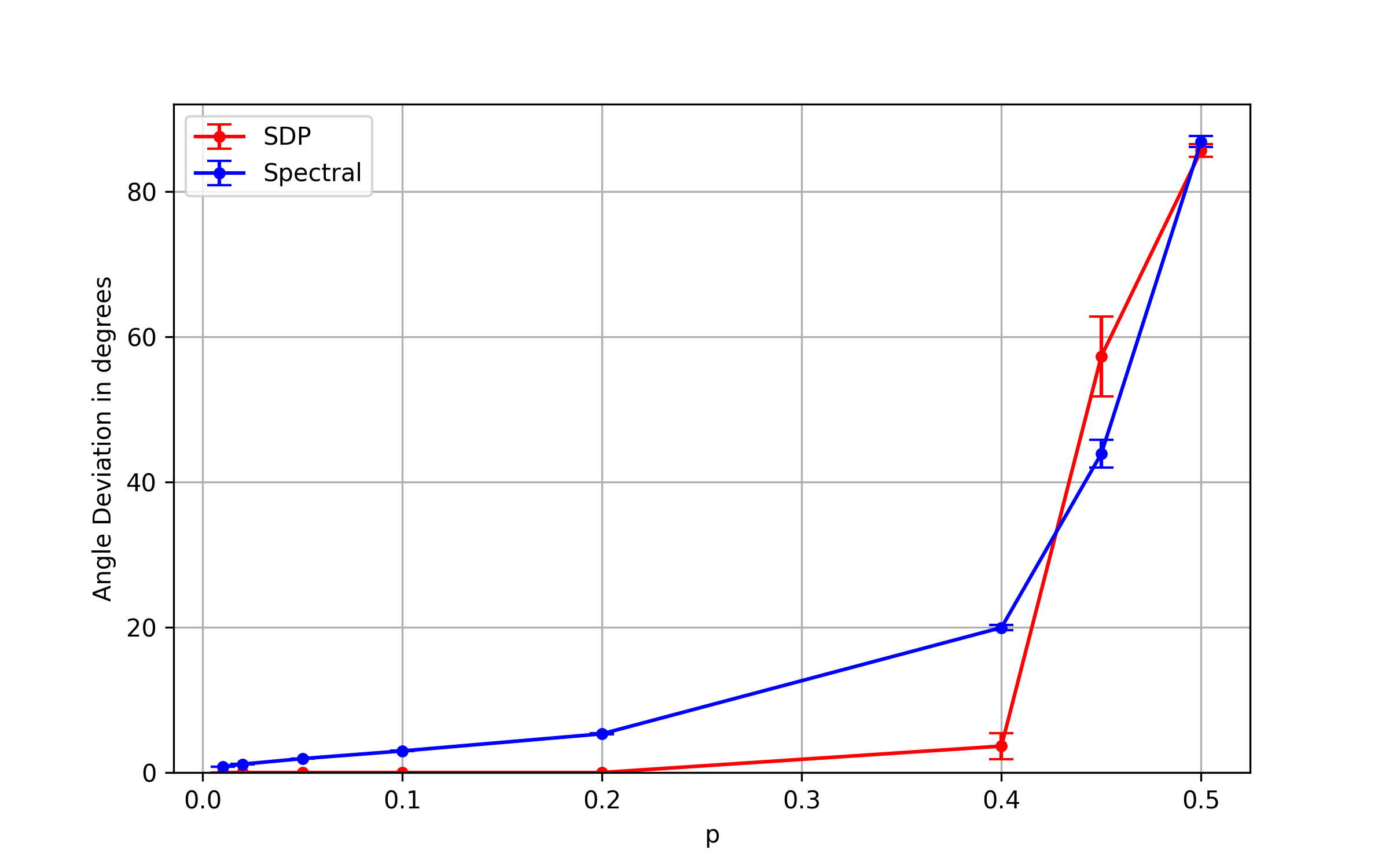}
        \caption{}
    \end{subfigure}
    \hfill
    \begin{subfigure}[t]{0.48\textwidth}
        \centering
        \includegraphics[width=\linewidth,trim={10 0 50 35},clip]{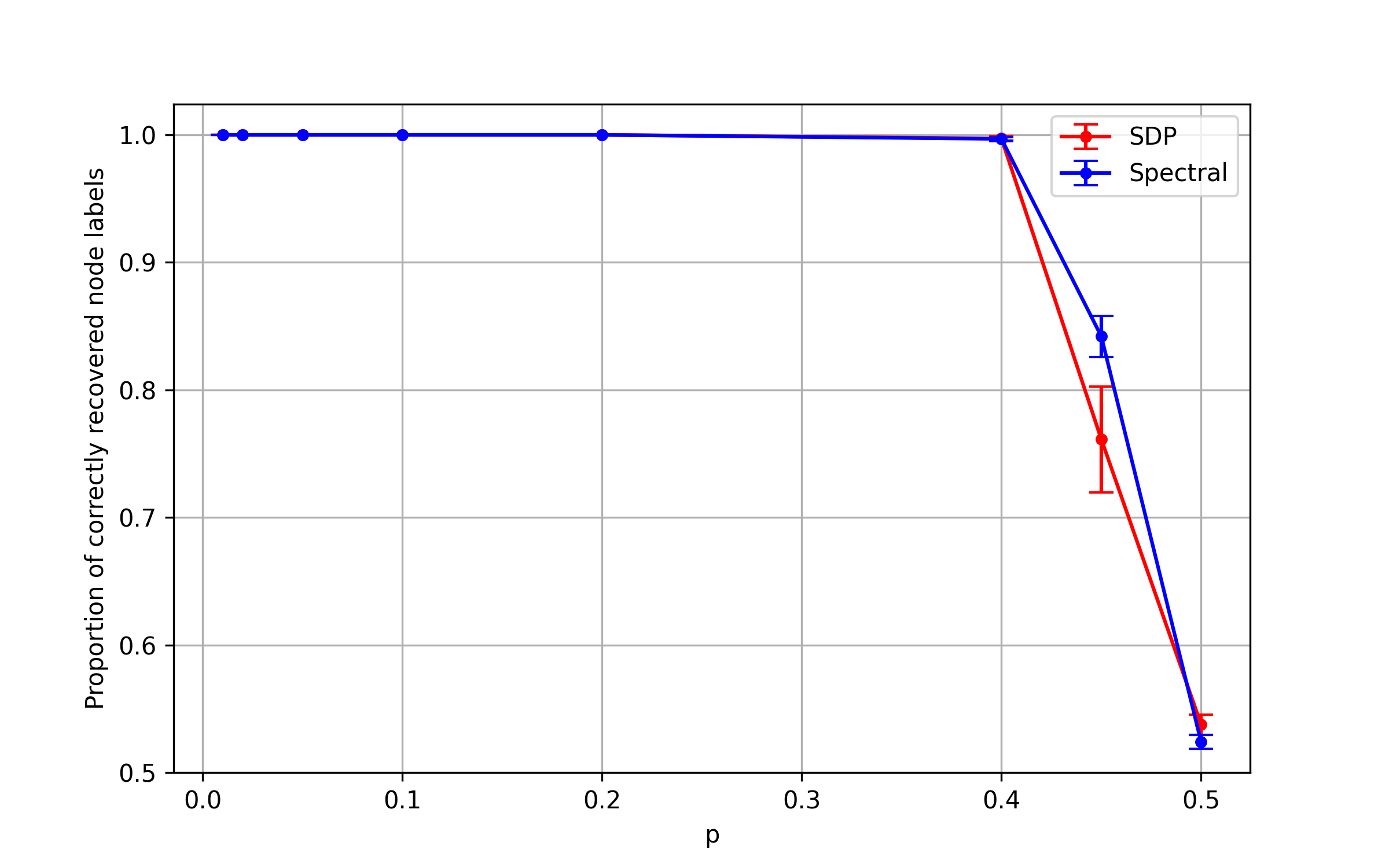}
        \caption{}
    \end{subfigure}
    \caption{Effect of edge-noise parameter $p$ on the angle deviation (a) and proportion of correctly recovered node labels (b) for complete graphs of $n=200$ nodes.
    Error bars at 95\% confidence level over 30 repetitions.
    Our spectral method (Spectral) and semidefinite programming (SDP) produced comparable results.}
    \label{fig:complete-p}
\end{figure}

\begin{figure}
    \centering
    \begin{subfigure}[t]{0.48\textwidth}
        \centering
        \includegraphics[width=\linewidth,trim={10 0 50 35},clip]{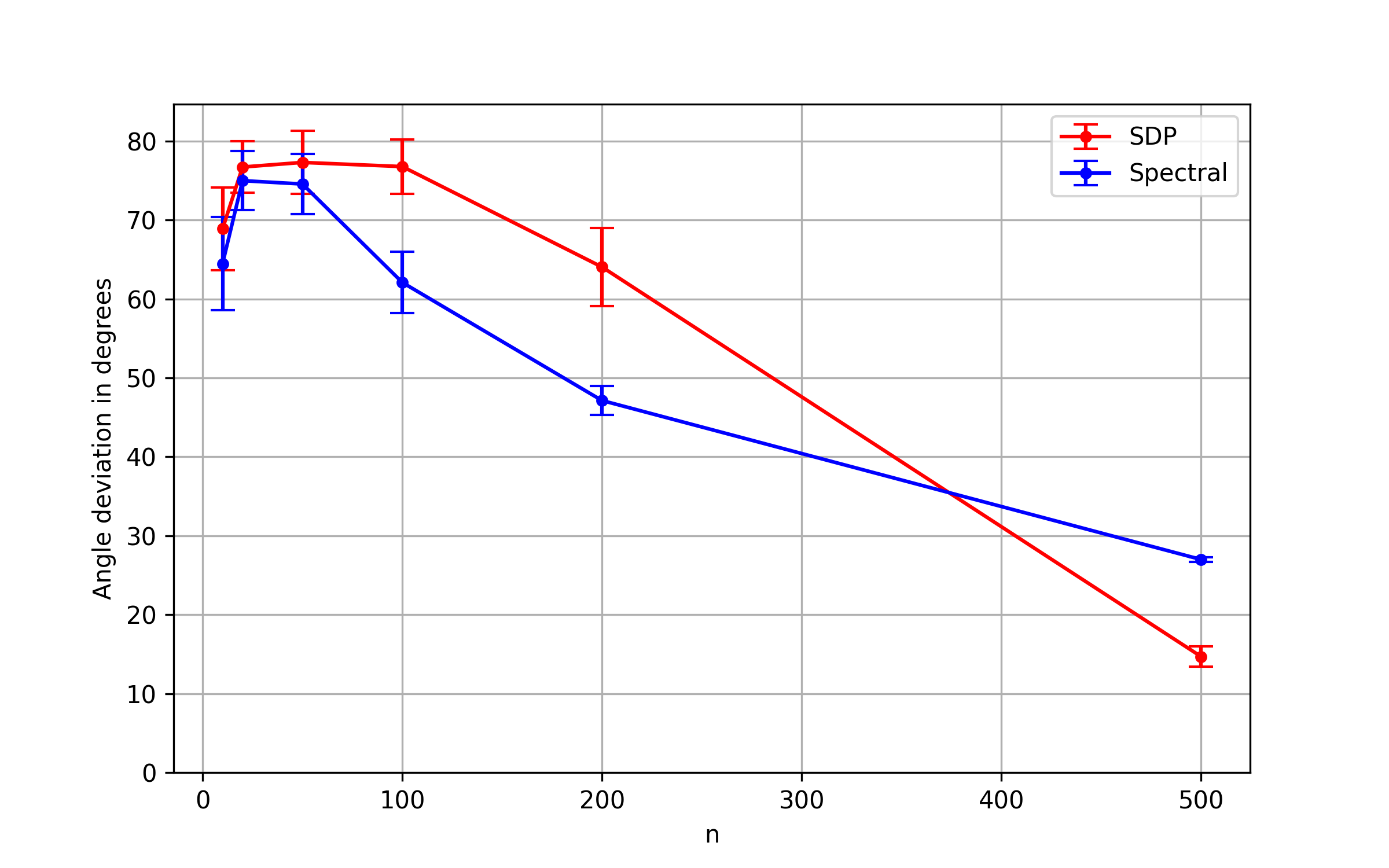}
        \caption{}
    \end{subfigure}
    \hfill
    \begin{subfigure}[t]{0.48\textwidth}
        \centering
        \includegraphics[width=\linewidth,trim={10 0 50 35},clip]{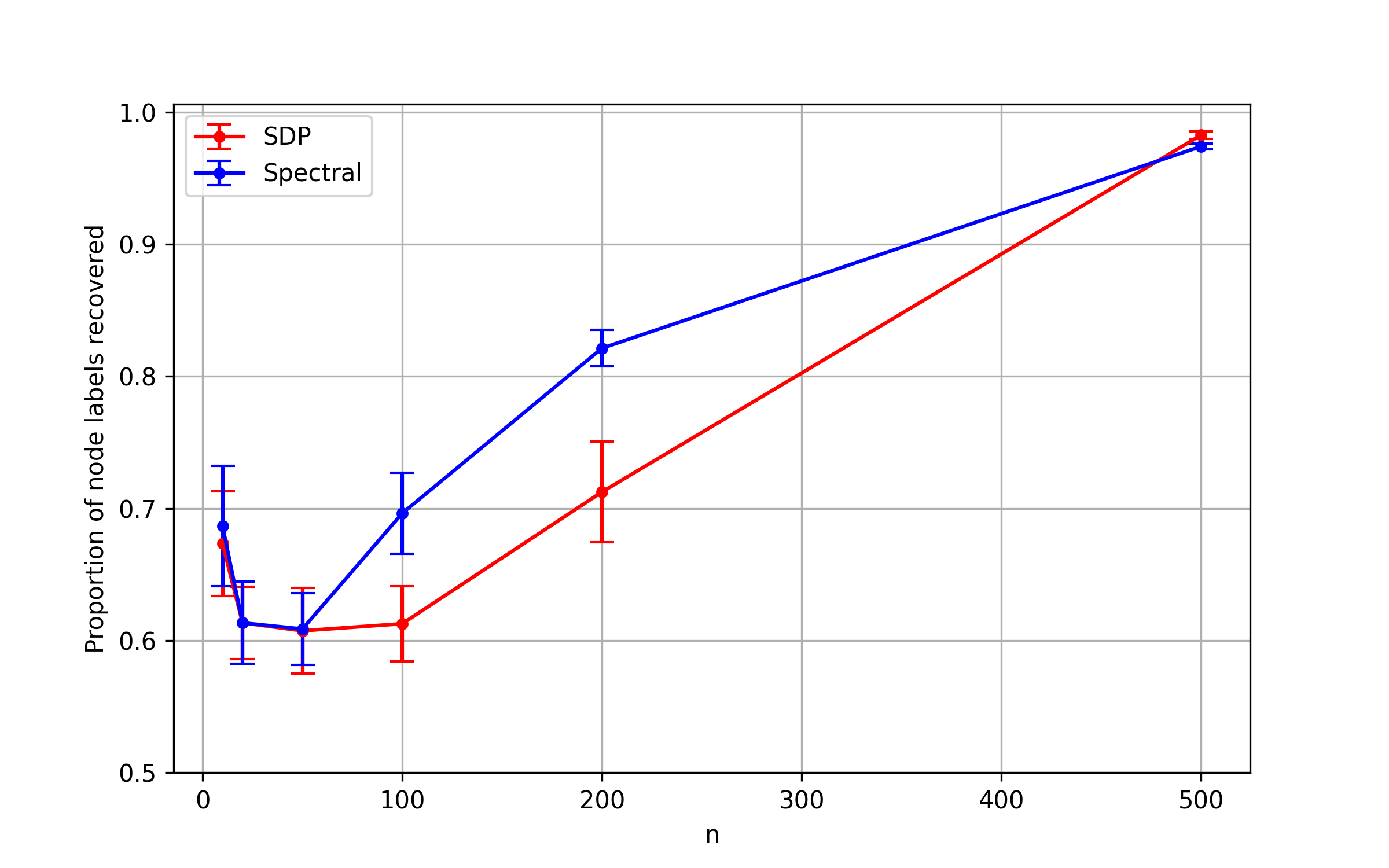}
        \caption{}
    \end{subfigure}
    \caption{Effect of number of nodes $n$ on the angle deviation (a) and proportion of correctly recovered node labels (b) for Erd\"{o}s–R\'{e}nyi graphs (60\% edge density) with $p=0.4$.
    Error bars at 95\% confidence level over 30 repetitions.
    Our spectral method (Spectral) and semidefinite programming (SDP) produced comparable results.}
    \label{fig:ER-n}
\end{figure}

\begin{figure}
    \centering
    \begin{subfigure}[t]{0.48\textwidth}
        \centering
        \includegraphics[width=\linewidth,trim={10 0 50 35},clip]{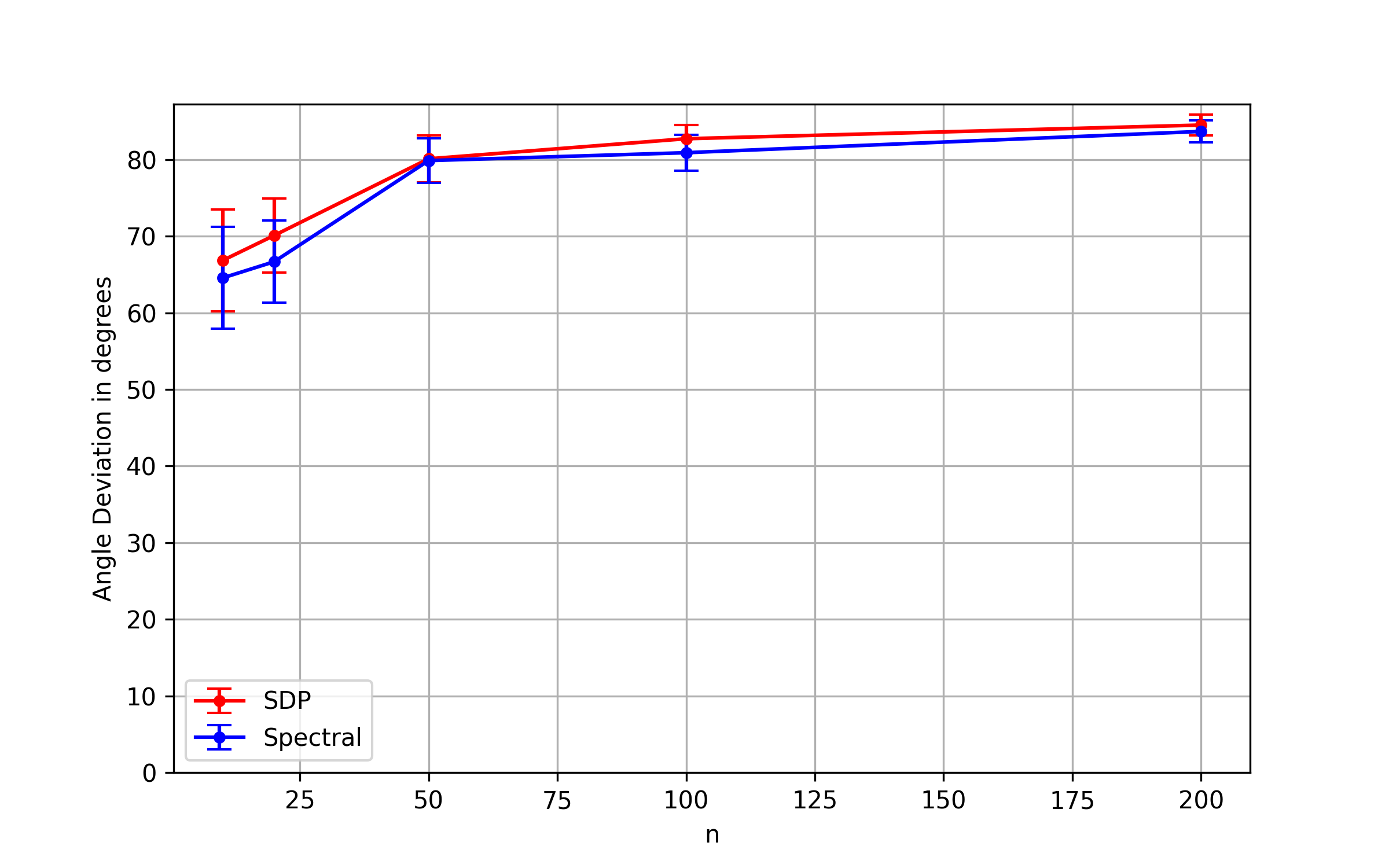}
        \caption{}
    \end{subfigure}
    \hfill
    \begin{subfigure}[t]{0.48\textwidth}
        \centering
        \includegraphics[width=\linewidth,trim={10 0 50 35},clip]{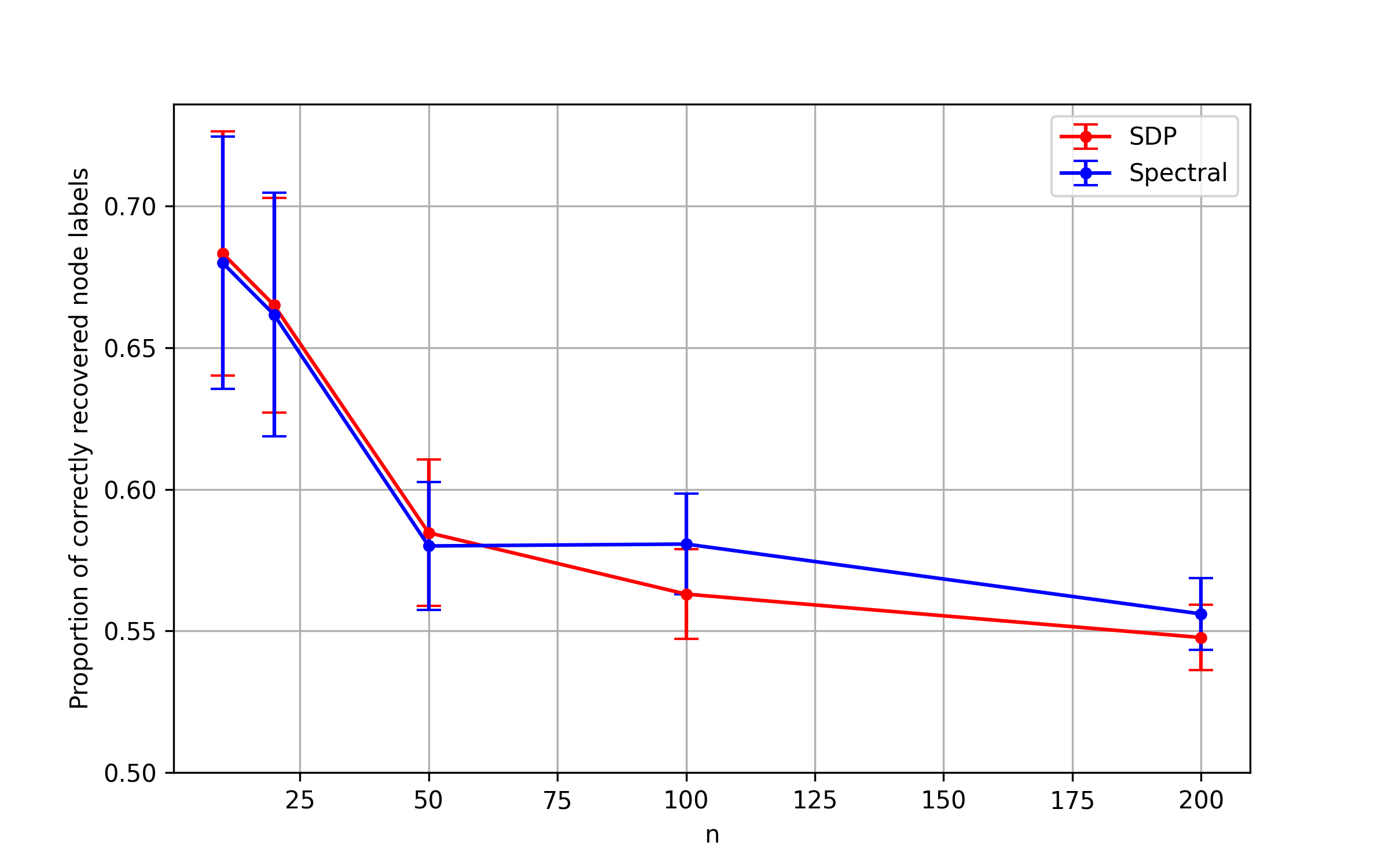}
        \caption{}
    \end{subfigure}
    \caption{Effect of number of nodes $n$ on the angle deviation (a) and proportion of correctly recovered node labels (b) for $d$-regular expanders ($d=6$) with $p=0.4$.
    Error bars at 95\% confidence level over 30 repetitions.
    Our spectral method (Spectral) and semidefinite programming (SDP) produced comparable results.}
    \label{fig:expander-n}
\end{figure}

\begin{figure}
    \centering
    \begin{subfigure}[t]{0.48\textwidth}
        \centering
        \includegraphics[width=\linewidth,trim={10 0 50 35},clip]{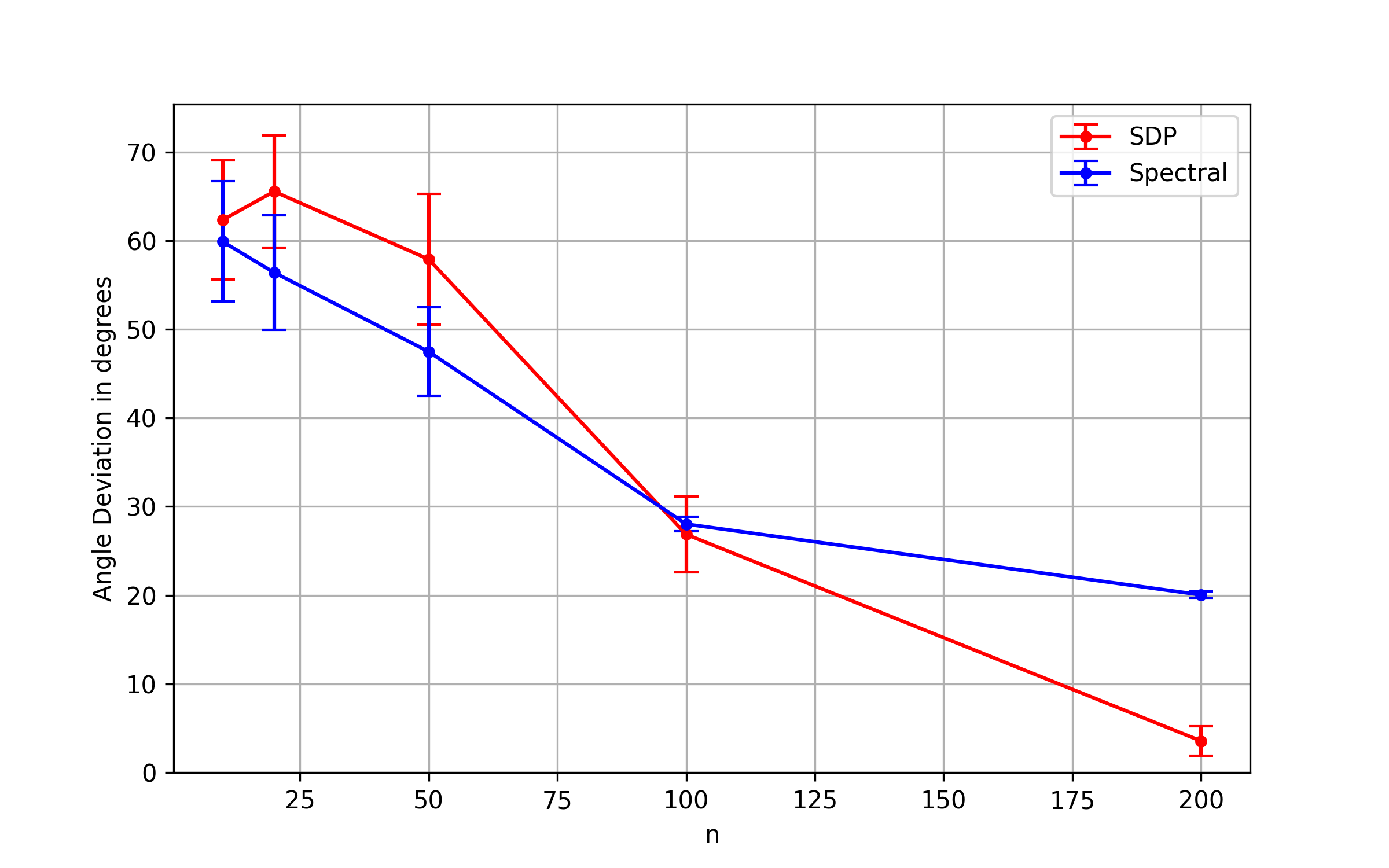}
        \caption{}
    \end{subfigure}
    \hfill
    \begin{subfigure}[t]{0.48\textwidth}
        \centering
        \includegraphics[width=\linewidth,trim={10 0 50 35},clip]{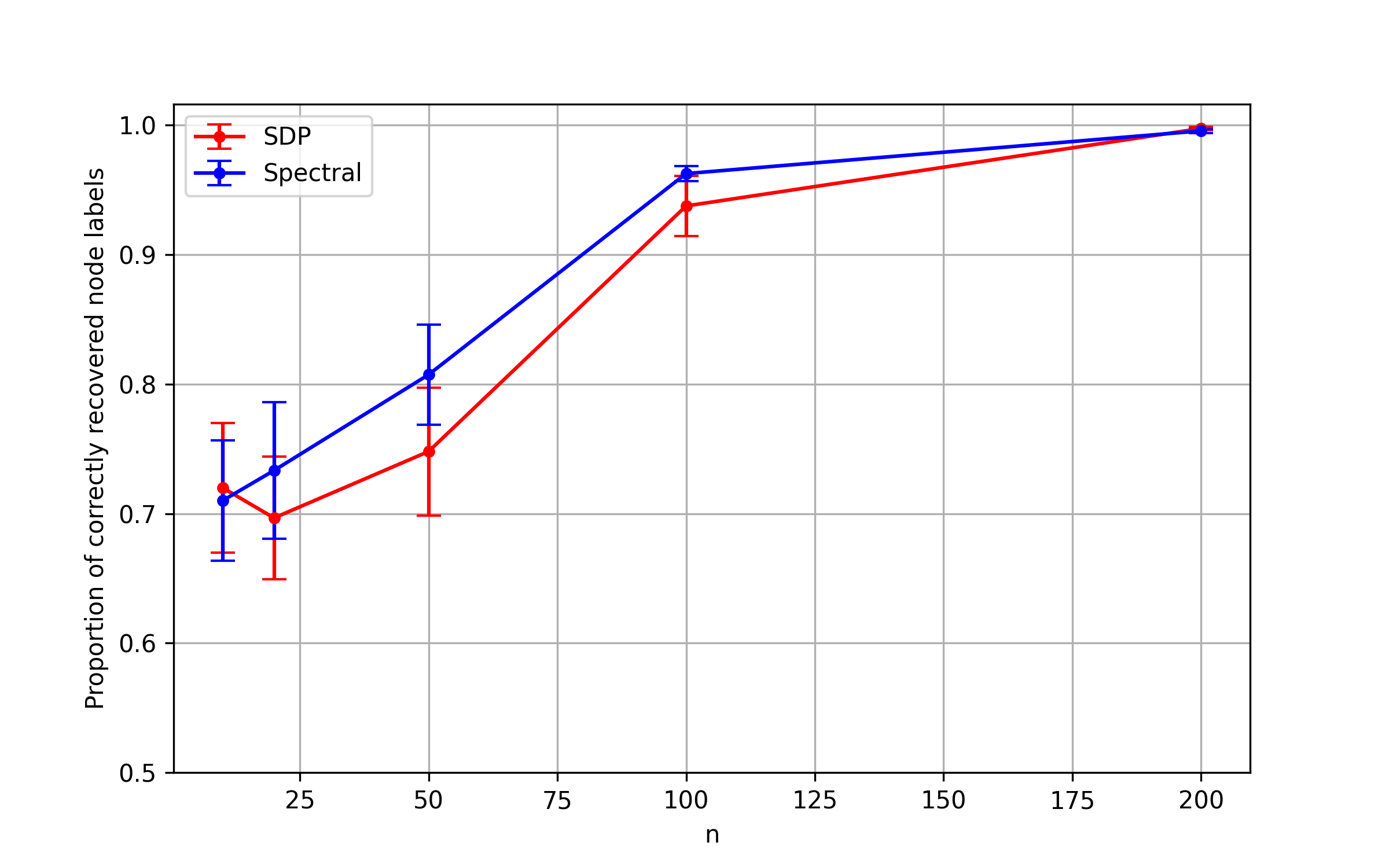}
        \caption{}
    \end{subfigure}
    \caption{Effect of number of nodes $n$ on the angle deviation (a) and proportion of correctly recovered node labels (b) for complete graphs with $p=0.4$.
    Error bars at 95\% confidence level over 30 repetitions.
    Our spectral method (Spectral) and semidefinite programming (SDP) produced comparable results.}
    \label{fig:complete-n}
\end{figure}

First, we run experiments on different classes of graphs, regarding the effect of the edge-noise parameter $p$ in the angle deviation and proportion of correctly recovered node labels.
\cref{fig:ER-p,fig:expander-p,fig:complete-p} show that our spectral method and semidefinite programming produce comparable results for Erd\"{o}s–R\'{e}nyi graphs, regular expander graphs and complete graphs.

Second, we run additional experiments for different classes of graphs, regarding the effect of the number of nodes $n$ in the angle deviation and proportion of correctly recovered node labels.
\cref{fig:ER-n,fig:expander-n,fig:complete-n} show that our spectral method and semidefinite programming produce comparable results for Erd\"{o}s–R\'{e}nyi graphs, regular expander graphs and complete graphs.
Recall that in the main text, \cref{fig:runtime-comparison} showed that our spectral method runs significantly faster than semidefinite programming.

\subsection{Large Real-World Structured Prediction} \label{app:realworld}

We performed experiments in Epinions, a real-world case with 131,828 nodes and 841,372 edges. Our spectral algorithm runs in 0.1484 seconds while SDP would struggle at this scale in commercial solvers. To validate whether the spectral approach result is meaningful, the proportion positive edges (over all edges) inside nodes with $y_i=+1$ is 82.1\%. For nodes with $y_i=-1$ the proportion of positive edges is 94.3\%. The proportion of positive edges between nodes with $y_i=+1$ and nodes with $y_i=-1$ is 23.0\%. This result is meaningful as we expect more positive edges inside nodes with the same $y_i$ and few positive edges between different $y_i$.
(On each case the proportion of negative edges is 100\% minus the proportion of positive edges.)

\end{document}